%% file: iclr2027_conference.tex
\documentclass{article} 
\usepackage{iclr2027_conference,times}

\iclrfinalcopy

\input{math_commands.tex}

\usepackage{hyperref}
\usepackage{url}
\usepackage{booktabs}

\definecolor{linkred}{RGB}{180,40,40}
\definecolor{linkblue}{RGB}{25,85,165}

\hypersetup{
    colorlinks=true,
    linkcolor=linkred,
    citecolor=linkblue,
    urlcolor=linkblue
}

\usepackage[table]{xcolor}
\usepackage{tcolorbox}
\tcbuselibrary{skins,breakable}
\usepackage{tocloft}

\definecolor{titlebgcolor}{RGB}{82,140,182}
\definecolor{textbgcolor}{RGB}{219,232,241}

\definecolor{tablegreen}{RGB}{46,125,50}

\newtcolorbox{infobox}[1][]{
    colframe=titlebgcolor!60,
    colback=textbgcolor!29,
    fontupper=\normalfont,
    #1
}

\newtcolorbox{titlebox}[2][]{
    colframe=titlebgcolor!60,
    colback=textbgcolor!29,
    colbacktitle=titlebgcolor!25,
    coltitle=black,
    title={#2},
    fontupper=\normalfont,
    fonttitle=\bfseries,
    arc=2pt,
    left=6pt,
    right=6pt,
    top=4pt,
    bottom=4pt,
    toptitle=3pt,
    bottomtitle=3pt,
    before skip=6pt,
    after skip=6pt,
    #1
}

\usepackage{amsmath,amssymb,amsthm}

\newtheorem{proposition}{Proposition}
\newtheorem{corollary}{Corollary}
\newtheorem{lemma}{Lemma}

\usepackage{algorithm}
\usepackage{algpseudocode}

\title{Does Every User Need a Private LoRA? \\ Decoupling Personalization from Per-User Adaptation}

\author{
Songyuan Sui\textsuperscript{1,2}\thanks{
Work done during an internship at Samsung Semiconductor, US. Corresponding author.}
\quad
Srikanth Malla\textsuperscript{1}
\quad
Chiho Choi\textsuperscript{1}
\quad
Joon Hee Choi\textsuperscript{1}
\\[1mm]
\textsuperscript{1}Samsung Semiconductor, US
\qquad
\textsuperscript{2}Rice University
\\
\texttt{Songyuan.Sui@rice.edu}
\\
\texttt{\{srikanth.m, chiho1.choi, jh4.choi\}@samsung.com}
}

\begin{document}

\maketitle

\fancyhead{}
\lhead{\textit{Preprint}}

\begin{abstract}
Personalized large language models often require a complete adaptation state for each user. However, this paradigm scales poorly as the user population grows. We revisit this design through the lens of \emph{personalization capacity allocation}: how much adaptation capacity can be shared across users, how that shared capacity should be composed, and how much adaptation must remain user-specific. We answer these questions through three complementary empirical analyses. We find that independent user adapters contain substantial cross-user reusable structure, that the utility of reusable directions reflects both user relevance and variation across queries, and that user histories provide transferable signals for compact individual correction. Motivated by these findings, we propose \textbf{LINEUP}. It learns a bank of reusable low-rank personalization factors, composes them through user-conditioned recall and query-dependent calibration, and restricts target-user adaptation to a tiny user code over a shared correction space. This design decouples expressive personalization capacity from per-user trainable state. Each target user optimizes only eight scalars, while all shared components remain fixed. By comparison, the evaluated private-LoRA configuration uses 4.19 million per-user parameters. Our theoretical analysis gives a finite-step, finite-history risk bound and sufficient conditions for user-code refinement to improve on history initialization. Across six tasks spanning personalized classification, prediction, and generation, LINEUP leads on all 12 metrics, each averaged over three independent runs (e.g., reducing LaMP-3 RMSE by 11.4\% relative to the strongest baseline). It maintains advantages under limited history. These results show that rich personalization can be supported primarily by reusable, conditionally composed shared capacity, while independent user adaptation remains confined to a tiny correction state.
\end{abstract}

\addtocontents{toc}{\protect\setcounter{tocdepth}{-1}}

\section{Introduction}
\label{sec:introduction}

Large language models (LLMs) can acquire task-specific knowledge and behaviors through adaptation \citep{JMLR:v25:23-0870, wang2025parameter}. However, such adaptation is typically learned from data pooled across users and produces a single model shared by the population. It captures what users have in common within a task, but not how their preferences differ. For example, the same input may require different outputs for different users. Personalized Large Language Models (PLLMs) therefore seek to adapt shared models to individual preferences inferred from user histories or feedback \citep{liu2025surveypersonalizedlargelanguage}. Recently, PLLMs have received attention in scenarios across text generation, conversation, multimodal generation, and recommender systems \citep{trajanovski-etal-2021-text, zhao-etal-2025-personalens, 10.1145/3696410.3714843, 10.1145/3773966.3777954}.

Existing PLLM methods broadly differ in how user information is converted into personalized behavior. Tuning-free approaches incorporate user histories through summaries, retrieval, or inference-time steering \citep{richardson2023integratingsummarizationretrievalenhanced, 10.1145/3626772.3657783, lv2026costeercollaborativedecodingtimepersonalization}. Their effectiveness depends on the quality of history-derived signals. Learned adaptation-based methods instead encode personalization into parameters or representations. One direction performs \textbf{individualized adaptation construction}: OPPU directly optimizes a user-specific PEFT module \citep{tan-etal-2024-democratizing}, PerFit retains a personalized representation intervention alongside a collective shift \citep{ICLR2026_7a1ac10e}, and P2P uses a population-trained hypernetwork to generate complete personalized adapters from user information \citep{tan-etal-2026-instant}. Another direction improves reuse through \textbf{user-derived reusable components}, such as pieces, group-level modules, or anchor-derived PEFT modules \citep{tan-etal-2024-personalized, zhang-etal-2025-proper, li-etal-2026-mta}. These methods improve how individualized adaptation is obtained or reused, but largely retain user-associated adaptation modules as the starting point for organizing personalization. A detailed review of related work is provided in Appendix~\ref{app:related_work}.

We instead view personalized adaptation through the lens of \emph{personalization capacity allocation}. Rather than starting from a complete individualized adaptation state, we ask what personalization capacity can be learned directly for cross-user reuse, how this shared capacity should be composed for each request, and how much still must be independently estimated for a target user. This leads to the central question of our work: \textbf{How much of a personalized adapter is actually personal?}

We investigate this question through three complementary empirical analyses. \textbf{1.} Independently learned user adapters contain substantial cross-user reusable structure (Figure~\ref{fig:main_analysis}(a)). \textbf{2.} The functional utility of reusable directions exhibits both persistent user association (Figure~\ref{fig:main_analysis}(b)) and substantial variation across queries from the same user. \textbf{3.} User histories provide transferable individual correction signals (Figure~\ref{fig:main_analysis}(c)) that can be preferentially retained in compact shared spaces. Together, these findings suggest organizing personalization around shared capacity, conditional composition, and compact individual correction.

Motivated by these findings, we propose \textbf{LINEUP}, which decouples personalization capacity from per-user adaptation. LINEUP directly learns fine-grained low-rank personalization factors across all source-training users as reusable population-level capacity, rather than deriving reusable components from independently learned user adapters. A target user's history recalls a set of relevant factors. Each query from the same user then calibrates their sparse composition. Remaining individual variation is handled by shared correction maps controlled by a single eight-dimensional user code. The code is initialized from the user's history and refined using history-derived support examples while all shared components remain fixed. Therefore, rich personalization capacity is learned predominantly through reusable shared structure, while target-user optimization is confined to \textbf{only eight scalars}. By comparison, the evaluated OPPU configuration uses 4.19 million per-user parameters. To characterize this compact adaptation, we derive a finite-step, finite-history risk bound for user-code refinement, linking shared-space approximation and history initialization to fitting progress and support--query disagreement.

Across six tasks spanning personalized classification, prediction, and generation, LINEUP leads on all 12 reported metrics, with each reported score averaged over three independent runs. Under cold-start settings, LINEUP maintains clear advantages as the available user history is limited. The scalability analysis shows low incremental per-user storage as the user population grows.

In summary, our main contributions are as follows:
\begin{itemize}
    \item We systematically study personalized adaptation from a personalization capacity allocation perspective and identify three properties: cross-user reuse, user- and query-dependent composition, and transferable individual correction.

    \item We derive a finite-step, finite-history risk bound for regularized user-code refinement and establish sufficient conditions for improvement over history initialization.
    
    \item We propose \textbf{LINEUP}, which decouples personalization capacity from per-user adaptation. It achieves the best mean performance on various personalization tasks while reducing parameter and storage growth, with further benefits under cold-start settings.
\end{itemize}

\section{Problem Setup and Preliminaries}
\label{sec:prelinimary}

\subsection{Personalization Problem Setup}
\label{sec:personalized_problem_setup}

We study each personalized task $\mathcal{T}$ independently, with input space $\mathcal{X}_{\mathcal{T}}$ and output space $\mathcal{Y}_{\mathcal{T}}$. Let $\mathcal{U}_{\mathrm{src}}$ and $\mathcal{U}_{\mathrm{tgt}}$ denote disjoint sets of source and target users. Source users provide data for learning population-shared personalization components, while target users are unseen during shared training and are personalized only from their own available history.

Each user $u$ has an available history $\mathcal{H}_u=\{h_{u,i}\}_{i=1}^{n_u}$ and held-out queries $\mathcal{Q}_u=\{(x_{u,j},y_{u,j})\}_{j=1}^{m_u}$. Following LaMP and OPPU \citep{salemi-etal-2024-lamp, tan-etal-2024-democratizing}, history records are input--output pairs. Personalization may use $\mathcal{H}_u$ and the current query $x$, but never its target output $y$.

Let $\theta_{\mathcal{T}}$ denote the task-adapted LLM backbone, $\phi_{\mathcal{T}}$ the population-shared personalization parameters, and $\omega_u$ the target-user-specific adaptation state. We write personalized prediction abstractly as
$p_u(y\mid x,\mathcal{H}_u)
=
p_{\theta_{\mathcal{T}},\phi_{\mathcal{T}}}
(y\mid x,\mathcal{H}_u,\omega_u)$. This notation separates shared personalization capacity from the state that must be independently adapted for a target user, while allowing the available history and current query to directly condition prediction.

When target-user optimization is required, task-specific support examples $\mathcal{S}_u$ are constructed from the available history. We denote by $n_u^{\mathrm{sup}}=|\mathcal{S}_u|$ the number of support
examples used for fitting. It need not equal the number of history records used for history conditioning. Under limited-history evaluation, all target-user-dependent operations only use the currently available history, including history conditioning and support construction. The held-out evaluation queries remain fixed.

\subsection{Adaptation Preliminaries}
\label{sec:lora_preliminaries}

\noindent\textbf{Parameter-space adaptation.} LoRA \citep{hu2021loralowrankadaptationlarge} augments a frozen linear transformation with a low-rank update. At an adapted site $i\in\mathcal{I}$, a user-specific LoRA induces the effective update $D_{u,i}=s_iB_{u,i}A_{u,i}$, where
$A_{u,i}\in\mathbb{R}^{r_{\mathrm{L}}\times d_{\mathrm{in},i}}$, $B_{u,i}\in\mathbb{R}^{d_{\mathrm{out},i}\times r_{\mathrm{L}}}$, $r_{\mathrm{L}}\ll
\min(d_{\mathrm{in},i},d_{\mathrm{out},i})$, and $s_i$ is a scaling coefficient. OPPU independently fits these
user-specific updates on a common task-adapted LLM backbone, after which the same fitted update is reused across all that user's queries. Our parameter-space analyses operate on the effective updates $D_{u,i}$, rather than on a particular factorization into $A_{u,i}$ and $B_{u,i}$.

\noindent\textbf{Representation-space correction.}
We also study personalization directly in representation space. Let $h_{\ell,t}\in\mathbb{R}^{d}$ denote the hidden representation at layer $\ell$ and token position $t$. An additive intervention takes the form
$h'_{\ell,t}=h_{\ell,t}+a_{\ell,t}v_\ell$, where
$v_\ell\in\mathbb{R}^{d}$ is a correction vector and $a_{\ell,t}$ specifies its positional scope. For a fixed reference model and example $z$, let $\mathcal{L}_{\mathcal{T}}(z;v_\ell)$ denote the negative log-likelihood under this intervention, and define the first-order loss-reducing direction as
$g_z^\ell=
-\left.\nabla_{v_\ell}
\mathcal{L}_{\mathcal{T}}(z;v_\ell)\right|_{v_\ell=0}$.
Because $v_\ell$ is unrestricted, this view provides a diagnostic of the representation-space correction indicated by the example's loss gradient, before imposing a low-dimensional user parameterization. Section~\ref{sec:analysis} specifies the concrete intervention and normalization used in the corresponding empirical analysis.

\begin{figure}[t]
    \centering
    \begin{minipage}[t]{0.37\linewidth}
        \centering
        \includegraphics[height=3.6cm]{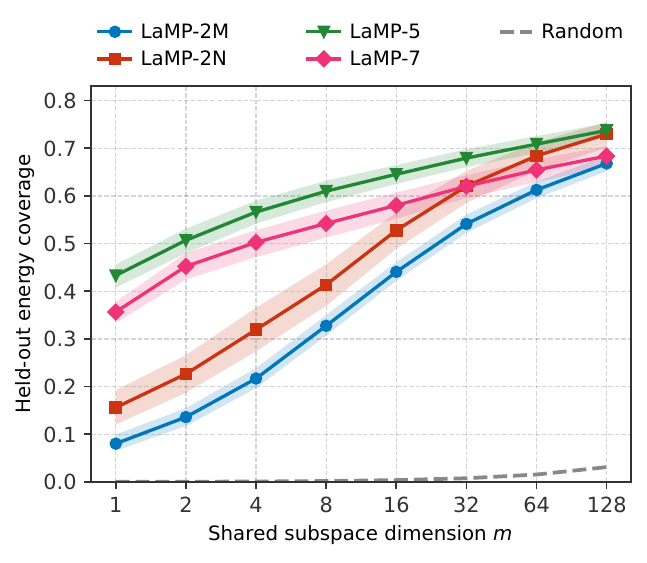}\\[-1mm]
        (a) Shared-subspace coverage
    \end{minipage}\hfill%
    \begin{minipage}[t]{0.25\linewidth}
        \centering
        \includegraphics[height=3.6cm]{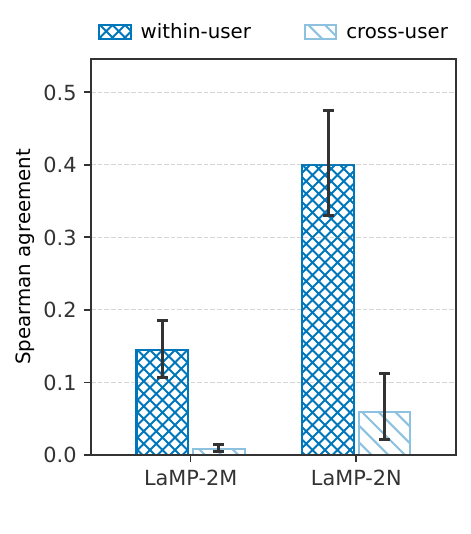}\\[-1mm]
        (b) Rank agreement
    \end{minipage}\hfill%
    \begin{minipage}[t]{0.35\linewidth}
        \centering
        \includegraphics[height=3.6cm]{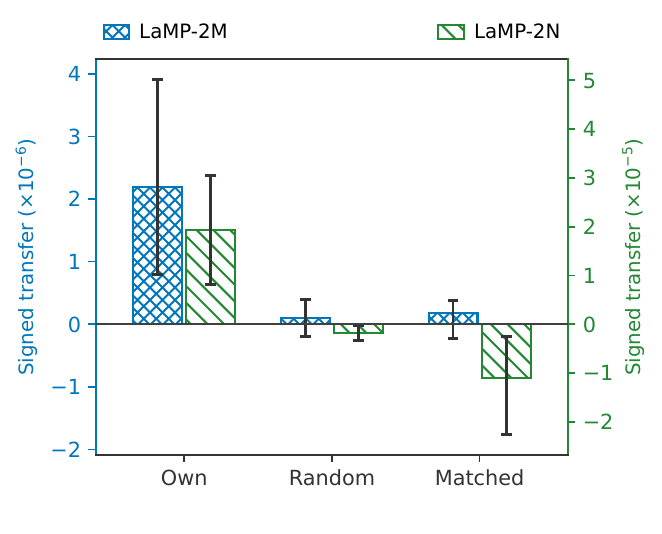}\\[-1mm]
        (c) Support-to-query transfer
    \end{minipage}

    \caption{Empirical evidence for personalization capacity allocation.
    (a) Small shared subspaces capture substantial held-out adapter-update energy.
    (b) Direction-utility rankings are substantially more consistent within users than across users.
    (c) A user's own history yields positive transferable correction, unlike random or label-matched donor histories (separate $y$-axis scales).}
    \label{fig:main_analysis}
\end{figure}

\section{Empirical Analysis of Personalized Adaptation}
\label{sec:analysis}

We examine how personalization capacity can be shared, composed, and individually corrected. RQ1 and RQ2 analyze independently optimized OPPU adapters \citep{tan-etal-2024-democratizing} on a fixed LLM backbone; RQ3 studies representation corrections without user adapters. Unless stated otherwise, intervals are 95\% user-level bootstrap confidence intervals. See Appendix~\ref{app:analysis} for detailed mathematical definitions and supporting evidence. Dataset details are provided in Section~\ref{sec:experimental_setup}. The design choices informed by these analyses are evaluated through ablations in Section~\ref{sec:ablation}.

\subsection{RQ1: How Much Structure Is Reusable across User Adapters?}
\label{sec:analysis_sharing}

\begin{infobox}
\phantomsection\label{obs:sharing}
\textbf{Observation 1 (Sharing).}
Independently trained user adapters share substantial structure: a small subspace learned from other users captures much of held-out update energy.
\end{infobox}

We here learn common output directions from one group of users and test how much of other users' adapter-update energy they capture. We compare these directions with random subspaces of the same dimension on LaMP-2M, LaMP-2N, LaMP-5, and LaMP-7 datasets.

Figure~\ref{fig:main_analysis}(a) shows that 32 directions (0.8\% of the output dimension) capture 54.1\%--67.9\% of held-out update energy, versus 0.78\% for random subspaces. Removing task-adapter output directions barely changes coverage. Size-matched random-partition bases also capture substantial energy, but genuine-user bases cover 14.3--21.0 percentage points more on the same updates (Appendix~\ref{app:analysis_sharing}). These findings motivate jointly learning reusable personalization factors rather than duplicating shared structure in user adapters.

\subsection{RQ2: How Does Direction Utility Vary across Users and Queries?}
\label{sec:analysis_composing}

\begin{infobox}
\phantomsection\label{obs:composing}
\textbf{Observation 2 (Composing).}
The utility of reusable directions exhibits persistent user-associated structure and substantial variation across queries from the same user.
\end{infobox}

Are reusable directions equally useful for every request? On LaMP-2M and LaMP-2N, we keep each user's adapter fixed and measure the first-order change in query loss induced by strengthening each of eight shared directions. Comparing these signed responses within and across users tests whether their utility reflects recurring user patterns, request-specific needs, or both.

Figure~\ref{fig:main_analysis}(b) shows higher within-user than cross-user Spearman agreement: 0.145 versus 0.008 on LaMP-2M and 0.400 versus 0.059 on LaMP-2N. This within-user advantage holds across all analyzed layers and projections (Appendix~\ref{app:rq2_agreement}). Yet differences between a user's queries still account for 82.7\% and 61.3\% of response variation after controlling for overall sensitivity scale (Appendix~\ref{app:rq2_variation}).
Direction-removal checks further show that the sensitivities closely track actual loss changes (Appendix~\ref{app:rq2_removal}). Together, these results suggest a hierarchical technique where history-conditioned candidate recall is followed by query-dependent calibration.

\subsection{RQ3: Can User History Guide Compact Individual Correction?}
\label{sec:analysis_correcting}

\begin{infobox}
\phantomsection\label{obs:correcting}
\textbf{Observation 3 (Correcting).}
User histories provide transferable correction after removing average label-response directions, and compact shared spaces preferentially retain it.
\end{infobox}

We now ask whether correction inferred from a user's history transfers to new queries, measured by alignment with their first-order loss-reducing directions. On LaMP-2M and LaMP-2N, we control shared label structure by removing the span of average label responses from both history and query correction signals. We compare own history with random and label-matched donor histories and evaluate compact spaces learned only from other users' histories (Appendix~\ref{app:analysis_correcting}).

Figure~\ref{fig:main_analysis}(c) shows positive own-history transfer, with intervals excluding zero and means above those of both donor conditions. Own-history point estimates are positive at all eight analyzed layers on both tasks (Appendix~\ref{app:rq3_transfer}). First-order intervention checks support the loss-based interpretation of this alignment (Appendix~\ref{app:rq3_validation}). An eight-dimensional space per layer retains 34.7\% and 29.4\% of full-space own-history transfer on LaMP-2M and LaMP-2N, respectively, while capturing only 12.9\% and 17.0\% of residual query-correction energy (Figure~\ref{fig:appendix_analysis_3}(a)). Compact spaces therefore preserve a larger fraction of useful transfer than of correction energy. This motivates shared correction maps with compact, history-fitted user coordinates, complementing conditional factor composition.

\section{Methodology: LINEUP}
\label{sec:method}

\begin{figure}[t]
    \centering
    \includegraphics[width=0.98\linewidth]{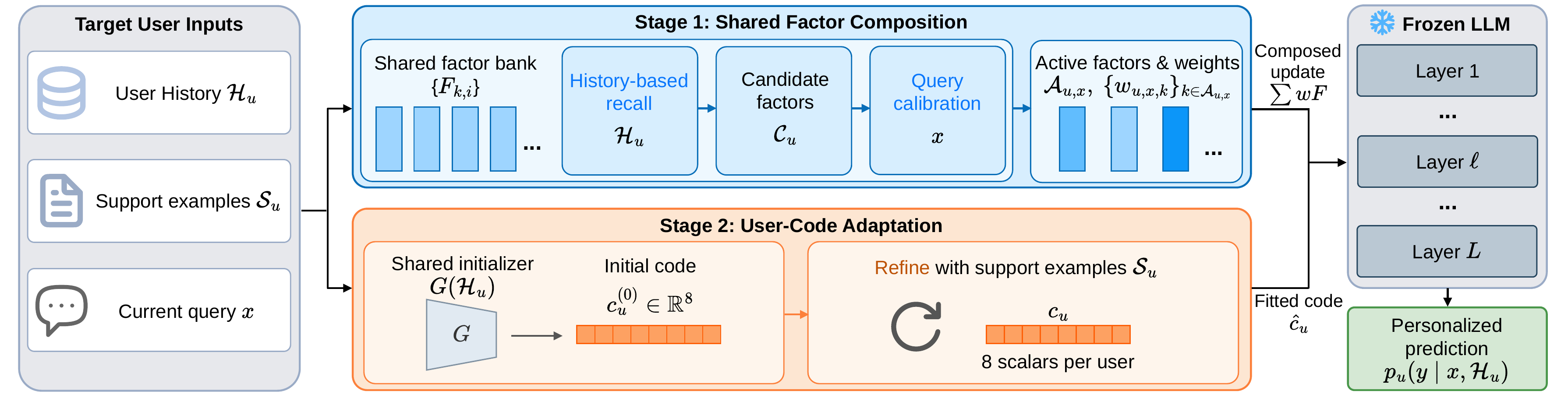}
    \caption{Overview of LINEUP.
    Stage 1 combines shared personalization factors
    (Section~\ref{sec:shared_factors}) through history-conditioned
    recall and query-dependent calibration (Section~\ref{sec:conditional_composition}).
    Stage 2 initializes an eight-dimensional user code from history
    and refines only this code on support examples,
    with Stage 1 active and all shared components fixed
    (Section~\ref{sec:tiny_user_codes}).}
    \label{fig:lineup}
\end{figure}

Motivated by findings in Section~\ref{sec:analysis}, we propose \textbf{LINEUP} (\textbf{LI}ghtweight perso\textbf{N}alization via r\textbf{EU}sable factors and tiny \textbf{P}er-user codes), which decouples personalization capacity from per-user adaptation. LINEUP comprises two stages (S1 and S2). Figure~\ref{fig:lineup} illustrates the architecture.

\begin{titlebox}{LINEUP
\hfill\hyperref[alg:lineup]{[Algorithm~\ref*{alg:lineup}]}}
\noindent\textbf{Stage 1 (S1): Shared-factor composition.}
User history recalls candidate factors, and the current query calibrates their sparse composition. Let $\mathcal A_{u,x}$ denote the active factors and $w_{u,x,k}$ their mixture weights. At each adapted linear map $i$, shared factors $F_{k,i}$ yield
\begin{equation}
D_{u,x,i}
=
\sum_{k\in\mathcal A_{u,x}}w_{u,x,k}F_{k,i}.
\label{eq:lineup_composition}
\end{equation}

\noindent\textbf{Stage 2 (S2): Tiny user-code adaptation.}
We refine a history-initialized code on support examples from the user history, with Stage 1 active and all shared parameters fixed. The fitted code $\hat c_u$ controls shared correction maps $V_\ell$ through
\begin{equation}
h'_{\ell,t}
=
h_{\ell,t}+a_{\ell,t}V_\ell\hat c_u.
\label{eq:lineup_correction}
\end{equation}
Here $h_{\ell,t}$ is the hidden state at layer $\ell$ and token $t$, and $a_{\ell,t}$ selects correction positions. One eight-dimensional code is reused across correction layers and the user's queries.
\end{titlebox}

\subsection{Fine-Grained Shared Personalization Factors}
\label{sec:shared_factors}

Observation \hyperref[obs:sharing]{1} motivates changing the unit of personalization from complete user adapters to reusable factors. LINEUP jointly learns its factor bank across all source-training users under conditional composition, rather than extracting components from separately trained user adapters. Each factor can serve multiple users without representing
any one user in full. This allows factors to specialize as complementary personalization units, while overlapping subsets produce distinct personalized updates (Appendix~\ref{app:lineup_composition}).

\subsection{Hierarchical History--Query Composition}
\label{sec:conditional_composition}

Observation \hyperref[obs:composing]{2} suggests separating user relevance from request-specific utility. Let $\tilde e_u,\tilde e_x\in\mathbb R^{d_f}$ denote fixed features of the available history and current query, obtained from frozen token embeddings (Appendix~\ref{app:lineup_composition}). For a bank of $K$ factors, learned shared maps $W_H,W_Q\in\mathbb R^{K\times d_f}$ implement history-conditioned recall followed by query calibration:
\begin{equation}
\begin{aligned}
s_u&=W_H\tilde e_u,
&\mathcal C_u&=\operatorname{TopK}(s_u,M_c),\\
z_{u,x}&=s_u+\lambda W_Q\tilde e_x,
&\mathcal A_{u,x}&=\operatorname{TopK}_{\mathcal C_u}(z_{u,x},M_a).
\end{aligned}
\label{eq:lineup_routing}
\end{equation}
Here $M_c$ and $M_a$ are the candidate and active set sizes, with $M_a\leq M_c < K$, and $\lambda$ controls query strength. $\operatorname{TopK}$ returns factor indices, and its subscript restricts selection to $\mathcal C_u$. History scores guide both candidate recall and query-calibrated composition. Softmax over active scores gives the mixture weights $w_{u,x,k}$ in Equation~\ref{eq:lineup_composition}, with zero weights outside $\mathcal A_{u,x}$. When $M_a=M_c$, query calibration reweights the recalled factors; when $M_a<M_c$, it can also change the active subset. S1 provides request-dependent personalization for unseen users without target-user gradient updates.

\subsection{History-Initialized Adaptation with Tiny User Codes}
\label{sec:tiny_user_codes}

Different users can draw on the same personalization factors while requiring different combinations of them. LINEUP therefore places its main personalization capacity in shared factor composition, while a tiny user code addresses individual differences that this composition does not fully capture. Observation \hyperref[obs:correcting]{3} motivates sharing correction capabilities while adapting only their user-specific coordinates. Shared maps learn how to correct representations. A single eight-dimensional code
specifies the user's correction. It is initialized from a population-informed starting point, and support fitting adjusts it using individual evidence. The same fitted code coordinates corrections across layers and is reused across queries, rather than introducing independently learned user matrices (Appendix~\ref{app:lineup_code_adaptation}).

\noindent\textbf{Learning shared capacity.}
We first train S1 across source users, then freeze it and learn the shared correction maps and initializer using initialized codes. S1 remains active during both shared S2 training and target-user code fitting. At personalization time, only the code is updated. The LLM backbone remains frozen throughout (Algorithm~\ref{alg:lineup}).

\subsection{Theoretical Analysis of Compact User Adaptation}
\label{sec:compact_user_theory}

We analyze S2's regularized, finite-step Adam fitting with S1 active and all shared parameters fixed. Under the conditions in Appendix~\ref{app:s2_theory}, our risk bound relates shared-space approximation and history-initialization error to fitting progress and support--query disagreement. It gives a sufficient condition for support fitting to improve on initialization, while accounting for history reuse (Corollary~\ref{cor:s2_finite_history}).

A deterministic bound on code displacement supports this finite-step analysis without requiring convergence (Lemma~\ref{lem:s2_adam_radius}). The risk bound separates shared-space approximation from code-space estimation, while accounting for shared geometry and stability. A complementary quadratic analysis explains when remaining correction need outweighs estimation noise (Corollary~\ref{cor:s2_effective_dimension}). Together, these results support learning rich correction capabilities jointly while keeping individual fitting compact.

\section{Experiments}
\label{sec:experiments}

\subsection{Experimental Setup}
\label{sec:experimental_setup}

\noindent\textbf{Datasets and metrics.} We evaluate LINEUP on six tasks of the LaMP personalization benchmark \citep{salemi-etal-2024-lamp}, spanning classification, rating prediction, and generation. We follow OPPU's data processing and target-user protocol, keeping source and target users disjoint. We report Accuracy and Macro-F1 for LaMP-2M/2N, mean absolute error (MAE) and root mean square error (RMSE) for LaMP-3, and ROUGE-1/ROUGE-L for generation. Lower is better for MAE and RMSE; higher is better otherwise. Dataset and metric details appear in Appendix~\ref{app:datasets}--\ref{app:metrics}.

\noindent\textbf{Baselines.} We compare LINEUP against eight baselines. \textbf{Non} uses the pretrained LLM backbone without task adaptation or personalization. \textbf{Basic} applies task adaptation without per-user personalization. \textbf{RAG} augments the task-adapted model's input with retrieved user history. \textbf{OPPU} \citep{tan-etal-2024-democratizing} learns a private adapter for each user, while \textbf{PerFit} \citep{ICLR2026_7a1ac10e} learns collective and user-specific representation shifts. \textbf{Per-Pcs} \citep{tan-etal-2024-personalized} assembles reusable pieces of user adapters. \textbf{P2P} \citep{tan-etal-2026-instant} trains a shared
hypernetwork across source users to generate adapters from user profiles and retrieved history. \textbf{MTA} \citep{li-etal-2026-mta} merges anchor-user adapters before individual refinement. Details appear in Appendix~\ref{app:baseline_impl}.

\noindent\textbf{Implementation details.} We use Llama-2-7B as the base model to align with most baselines' released code. LINEUP builds on OPPU's non-personalized task-adapted backbone and learns shared
components separately. We report mean $\pm$ standard deviation across three seeds. Implementation details and hyperparameters are provided in Appendix~\ref{app:implementation}--\ref{app:hyperparameters}.

\begin{table}[t]
\setlength{\tabcolsep}{3.6pt}
\renewcommand{\arraystretch}{1.12}

\caption{Performance on LaMP-2M/2N/3.
\textbf{Acc} = Accuracy, \textbf{F1} = Macro-F1, \textbf{MAE} = mean absolute error, \textbf{RMSE} = root mean square error.
\underline{Underline} and \textbf{bold} denote the second-best and best mean results. $\Delta$ reports absolute improvements over the strongest baseline, with error reductions shown as positive. LINEUP leads on all metrics across classification and prediction.}
\label{tab:main_results_2n_2m_3}
\begin{center}
\begin{tabular}{@{}lcccccc@{}}
\toprule
& \multicolumn{2}{c}{\textbf{LaMP-2M}}
& \multicolumn{2}{c}{\textbf{LaMP-2N}}
& \multicolumn{2}{c}{\textbf{LaMP-3}} \\
\cmidrule(lr){2-3}
\cmidrule(lr){4-5}
\cmidrule(lr){6-7}

\textbf{Method}
& \textbf{Acc} $\uparrow$
& \textbf{F1} $\uparrow$
& \textbf{Acc} $\uparrow$
& \textbf{F1} $\uparrow$
& \textbf{MAE} $\downarrow$
& \textbf{RMSE} $\downarrow$ \\
\midrule

Non
& $.032 \pm .001$
& $.035 \pm .002$
& $.355 \pm .002$
& $.120 \pm .006$
& $1.342 \pm .016$
& $2.005 \pm .022$ \\

Basic
& $.470 \pm .001$
& $.423 \pm .001$
& $.783 \pm .002$
& $.536 \pm .002$
& $.221 \pm .009$
& $.491 \pm .009$ \\

RAG
& $.552 \pm .000$
& $.479 \pm .002$
& $.817 \pm .001$
& $\underline{.594} \pm .001$
& $.216 \pm .009$
& $.495 \pm .009$ \\

OPPU
& $.619 \pm .000$
& $.515 \pm .001$
& $.801 \pm .001$
& $.556 \pm .002$
& $.181 \pm .009$
& $.441 \pm .011$ \\

PerFit
& $\underline{.632} \pm .000$
& $\underline{.519} \pm .001$
& $\underline{.818} \pm .000$
& $.589 \pm .002$
& $\underline{.172} \pm .000$
& $\underline{.440} \pm .000$ \\

Per-Pcs
& $.503 \pm .001$
& $.458 \pm .002$
& NA
& NA
& $.217 \pm .000$
& $.490 \pm .000$ \\

P2P
& $.521 \pm .005$
& $.472 \pm .006$
& $.804 \pm .003$
& $.563 \pm .004$
& $.215 \pm .005$
& $.488 \pm .004$ \\

MTA
& $.526 \pm .005$
& $.478 \pm .005$
& NA
& NA
& $.204 \pm .005$
& $.468 \pm .004$ \\

\rowcolor{cyan!10}
\textbf{LINEUP}
& $\mathbf{.645} \pm .003$
& $\mathbf{.539} \pm .002$
& $\mathbf{.822} \pm .005$
& $\mathbf{.621} \pm .003$
& $\mathbf{.152} \pm .003$
& $\mathbf{.390} \pm .003$ \\

\rowcolor{cyan!10}
$\Delta$
& $\textbf{\textcolor{tablegreen}{+0.013}}$
& $\textbf{\textcolor{tablegreen}{+0.020}}$
& $\textbf{\textcolor{tablegreen}{+0.004}}$
& $\textbf{\textcolor{tablegreen}{+0.027}}$
& $\textbf{\textcolor{tablegreen}{+0.020}}$
& $\textbf{\textcolor{tablegreen}{+0.050}}$ \\

\bottomrule
\end{tabular}
\end{center}
\end{table}

\begin{table}[t]
\setlength{\tabcolsep}{3.6pt}
\renewcommand{\arraystretch}{1.12}

\caption{Performance on LaMP-4/5/7.
\textbf{R-1} = ROUGE-1, \textbf{R-L} = ROUGE-L. \underline{Underline} and \textbf{bold} denote the second-best and best mean results. $\Delta$ reports absolute improvements over the strongest baseline. LINEUP leads on all metrics across generation tasks.}
\label{tab:main_results_4_5_7}
\begin{center}
\begin{tabular}{@{}lcccccc@{}}
\toprule
& \multicolumn{2}{c}{\textbf{LaMP-4}}
& \multicolumn{2}{c}{\textbf{LaMP-5}}
& \multicolumn{2}{c}{\textbf{LaMP-7}} \\
\cmidrule(lr){2-3}
\cmidrule(lr){4-5}
\cmidrule(lr){6-7}

\textbf{Method}
& \textbf{R-1} $\uparrow$
& \textbf{R-L} $\uparrow$
& \textbf{R-1} $\uparrow$
& \textbf{R-L} $\uparrow$
& \textbf{R-1} $\uparrow$
& \textbf{R-L} $\uparrow$ \\
\midrule

Non
& $.137 \pm .000$
& $.120 \pm .000$
& $.156 \pm .015$
& $.128 \pm .012$
& $.131 \pm .001$
& $.127 \pm .001$ \\

Basic
& $.186 \pm .000$
& $.161 \pm .001$
& $.477 \pm .006$
& $.413 \pm .007$
& $.520 \pm .004$
& $.471 \pm .006$ \\

RAG
& $.198 \pm .001$
& $.175 \pm .000$
& $.504 \pm .006$
& $.444 \pm .004$
& $.561 \pm .004$
& $.514 \pm .002$ \\

OPPU
& $.216 \pm .000$
& $.192 \pm .001$
& $.535 \pm .011$
& $.466 \pm .010$
& $\underline{.564} \pm .003$
& $\underline{.522} \pm .003$ \\

PerFit
& $.207 \pm .001$
& $.183 \pm .001$
& $.524 \pm .001$
& $.446 \pm .002$
& $.526 \pm .007$
& $.474 \pm .008$ \\

Per-Pcs
& $.187 \pm .001$
& $.174 \pm .000$
& $.482 \pm .004$
& $.427 \pm .005$
& $.522 \pm .001$
& $.473 \pm .001$ \\

P2P
& $.213 \pm .001$
& $.190 \pm .001$
& $.518 \pm .009$
& $.453 \pm .009$
& $.548 \pm .005$
& $.499 \pm .002$ \\

MTA
& $\underline{.218} \pm .001$
& $\underline{.195} \pm .000$
& $\underline{.537} \pm .007$
& $\underline{.475} \pm .005$
& NA
& NA \\

\rowcolor{cyan!10}
\textbf{LINEUP}
& $\mathbf{.222} \pm .001$
& $\mathbf{.197} \pm .001$
& $\mathbf{.564} \pm .003$
& $\mathbf{.494} \pm .004$
& $\mathbf{.594} \pm .004$
& $\mathbf{.551} \pm .006$ \\

\rowcolor{cyan!10}
$\Delta$
& $\textbf{\textcolor{tablegreen}{+0.004}}$
& $\textbf{\textcolor{tablegreen}{+0.002}}$
& $\textbf{\textcolor{tablegreen}{+0.027}}$
& $\textbf{\textcolor{tablegreen}{+0.019}}$
& $\textbf{\textcolor{tablegreen}{+0.030}}$
& $\textbf{\textcolor{tablegreen}{+0.029}}$ \\

\bottomrule
\end{tabular}
\end{center}
\end{table}

\subsection{Main Results on LaMP}
\label{sec:main_results}

Tables~\ref{tab:main_results_2n_2m_3} and~\ref{tab:main_results_4_5_7} show that \textbf{LINEUP achieves the best mean performance on all 12 metrics across six tasks}. Representative gains include a $0.020$ F1 improvement over PerFit on LaMP-2M, a $0.050$ RMSE reduction on LaMP-3, and a $0.029$ improvement in ROUGE-L over OPPU on LaMP-7. While the strongest baseline varies across tasks, LINEUP consistently leads across tasks.

\noindent\textbf{Effectively organizing shared capacity around reusable factors.} OPPU and PerFit retain individual adaptation states, while Per-Pcs and MTA reuse components derived from user adapters. Our comparison also includes cross-user, history-conditioned learning through P2P's population-trained hypernetwork. \hyperref[sec:analysis_sharing]{RQ~1} indicates that separate states repeatedly devote limited capacity to common structure. LINEUP jointly learns factors for
user- and query-conditioned composition, allowing overlapping subsets to serve different users and requests without retraining. Its distinction lies in organizing shared capacity around reusable factors while restricting independent adaptation to compact correction. The consistent gains support this factor-centric capacity-allocation strategy as a whole, combining shared learning, conditional composition, and compact individual refinement.

\noindent\textbf{Reserving individual adaptation for compact correction.} \hyperref[obs:correcting]{RQ~3} reveals that user histories provide transferable correction signals that compact shared spaces can preferentially retain. Section~\ref{sec:compact_user_theory} shows how finite-step fitting can improve history initialization when its progress outweighs support--query disagreement, with statistical complexity controlled in the user-code space. LINEUP therefore places most modeling capacity in shared structures and reserves a lightweight code for targeted individual correction. The results support this division of labor: shared capacity and compact individual correction together yield stronger personalization without a full private adaptation state per user.

\subsection{Cold-Start Personalization}
\label{sec:few_history}

\begin{figure}[t]
    \centering

    \begin{minipage}[t]{0.32\linewidth}
        \centering
        \includegraphics[width=\linewidth]{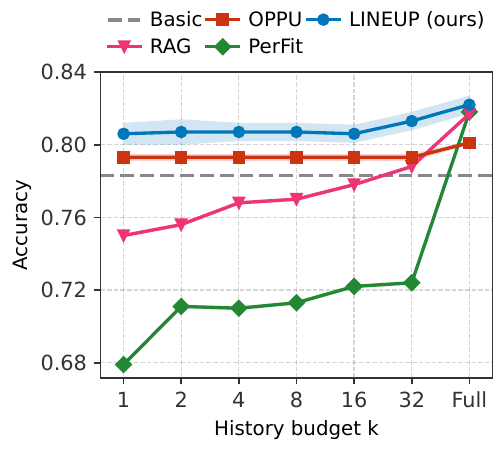}\\[-1mm]
        (a) LaMP-2N Accuracy
    \end{minipage}\hfill%
    \begin{minipage}[t]{0.32\linewidth}
        \centering
        \includegraphics[width=\linewidth]{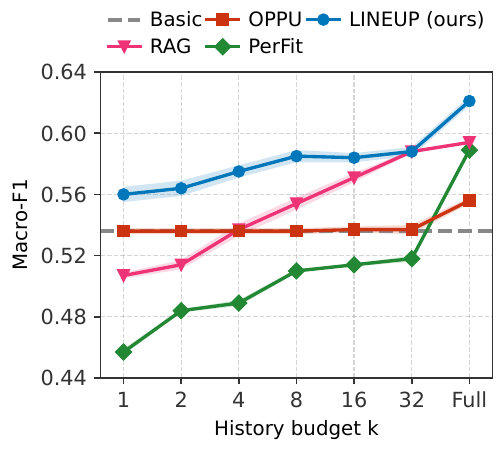}\\[-1mm]
        (b) LaMP-2N Macro-F1
    \end{minipage}\hfill%
    \begin{minipage}[t]{0.32\linewidth}
        \centering
        \includegraphics[width=\linewidth]{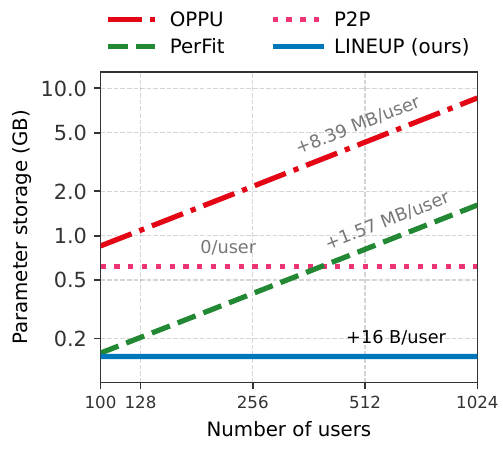}\\[-1mm]
        (c) LaMP-2N storage
    \end{minipage}

    \caption{Cold-start personalization and adaptation-storage scaling.
    (a) LINEUP maintains high classification accuracy.
    (b) Additional history improves Macro-F1 from an already strong cold-start result.
    (c) LINEUP has the lowest total adaptation storage over the displayed user range.
    }
    \label{fig:coldstart_efficiency}
\end{figure}

Cold-start scenarios are common in real-world PLLM deployment, where personalization must begin with limited user history. We therefore evaluate how effectively LINEUP personalizes unseen target users as the amount of available history is restricted. We vary the available history of unseen target users while keeping users and evaluation queries fixed. On LaMP-2N, LINEUP achieves the best mean Accuracy and the best or tied-best reported Macro-F1 across history budgets (Figure~\ref{fig:coldstart_efficiency}(a--b)). With only one history record, it reaches $0.806$ Accuracy and $0.560$ Macro-F1, compared with OPPU's $0.793/0.536$ at the same budget. Its Macro-F1 increases to $0.621$ with full history, showing that strong cold-start performance still benefits from additional user evidence. The detailed protocol and complementary generation task results on LaMP-4 appear in Appendix~\ref{app:coldstart_results}.

\noindent\textbf{Rich shared capacity from minimal history.} Our factor bank learns reusable personalization patterns across source users before target-user adaptation. Available history provides cues for recalling relevant factors, rather than requiring these patterns to be relearned for each user. It indicates that our shared factors already supply much of the needed personalization. The tiny user code then addresses residual individual differences not fully captured by factor composition. Concentrating independent estimation on this correction also aligns with the finite-history risk analysis in Section~\ref{sec:compact_user_theory}.

\subsection{Efficiency and Scalability}
\label{sec:efficiency}

Most benchmark tasks in our evaluation involve about 100 target users, whereas real-world PLLM services may serve far larger populations. Adaptation storage thereby becomes a vital constraint for multi-adapter deployment \citep{DBLP:journals/rts/AbdelzaherHKKMRTWWY25}. Figure~\ref{fig:coldstart_efficiency}(c) shows analytical storage scaling for the evaluated LaMP-2N configurations, normalized to bf16: LINEUP adds only 16\,B per user. LINEUP's total adaptation storage becomes lower than OPPU's at 17 users
and PerFit's at 95 users. Both thresholds are very small for large-scale services, where LINEUP's storage savings over these baselines grow with the user population. At 1,024 users, the projected totals, including shared parameters, are approximately 151.2\,MB for LINEUP, 1.61\,GB for PerFit, and 8.60\,GB for OPPU. Accounting details and comparisons with Per-Pcs and MTA appear in Appendix~\ref{app:efficiency}.

\noindent\textbf{Amortizing shared capacity across users.} In OPPU, each private adapter must accommodate both reusable patterns and individual variation, while overlapping structure is repeatedly learned and stored across users (\hyperref[sec:analysis_sharing]{RQ~1}). LINEUP instead invests in a larger factor bank jointly learned across source users, providing more capacity to capture diverse reusable personalization patterns. Only residual individual correction requires an eight-dimensional code. This organization explains how richer personalization and lower storage can coexist. Capacity is thus expanded where it can be shared, rather than replicated in every user module. MTA also reuses components but retains large individualized adapters, leaving substantial per-user storage growth (Figure~\ref{fig:coldstart_efficiency_supp}(c)).

\noindent\textbf{Compact per-request activation.} History-conditioned recall and query calibration allow each request to draw on only a small part of the shared bank. Each request activates four rank-4 factors from the 32-factor bank, yielding a combined update of rank at most 16 at each adapted projection. The larger bank therefore expands the range of reusable personalization patterns without requiring its full capacity to be activated for every request.

\begin{table}[t]
\setlength{\tabcolsep}{4.2pt}
\renewcommand{\arraystretch}{1.12}

\caption{Ablation results on LaMP-2M and LaMP-5. \textbf{Bold} denotes the best mean.}
\label{tab:ablation}
\begin{center}
\begin{tabular}{@{}lcccc@{}}
\toprule
& \multicolumn{2}{c}{\textbf{LaMP-2M}}
& \multicolumn{2}{c}{\textbf{LaMP-5}} \\
\cmidrule(lr){2-3}
\cmidrule(lr){4-5}

\textbf{Method}
& \textbf{Acc} $\uparrow$
& \textbf{F1} $\uparrow$
& \textbf{R-1} $\uparrow$
& \textbf{R-L} $\uparrow$ \\
\midrule

\rowcolor{cyan!10}
\textbf{Full LINEUP (S1 + S2)}
& $\mathbf{.645} \pm .003$
& $\mathbf{.539} \pm .002$
& $\mathbf{.564} \pm .003$
& $\mathbf{.494} \pm .004$ \\

S1 + initialized code
& $.639 \pm .002$
& $.533 \pm .001$
& $.555 \pm .000$
& $.490 \pm .001$ \\

S1 only
& $.634 \pm .005$
& $.530 \pm .002$
& $.549 \pm .003$
& $.485 \pm .001$ \\

S1 w/o query calibration
& $.585 \pm .008$
& $.494 \pm .007$
& $.523 \pm .002$
& $.461 \pm .002$ \\

S1 w/ random candidates
& $.497 \pm .023$
& $.419 \pm .030$
& $.492 \pm .012$
& $.436 \pm .011$ \\

\bottomrule
\end{tabular}
\end{center}
\end{table}

\subsection{Ablation Studies}
\label{sec:ablation}

We conduct an ablation study across representative tasks (LaMP-2M and LaMP-5) to validate our three design choices of LINEUP: history-guided factor retrieval, query-conditioned composition, and the tiny user-specific correction. The ablation protocol is provided in Appendix~\ref{app:ablation_protocol}. For the two routing ablations, S1-only is the reference. \emph{Random candidates} replaces history-guided recall with a fixed random subset of the same size for each user, while preserving history- and query-based scoring within the candidate set. \emph{No query calibration} removes the query contribution during training and inference while retaining history-conditioned recall. To examine S2, we consider three variants. \emph{S1 only} omits the entire correction mechanism, while \emph{S1 + initialized code} applies history-initialized correction without support fitting. Full LINEUP uses the same shared checkpoints and further refines the code on history-derived support examples.

Table~\ref{tab:ablation} shows that random candidates cause the largest degradation relative to S1-only, while removing query calibration also reduces performance on both tasks. Full LINEUP outperforms S1-only, supporting S2 as a whole. On all four metrics, shared correction maps with history-based initialization improve mean performance over S1-only, and support fitting adds further gains. These results support both routing choices within the factor bank and refinement beyond initialized correction. A code-dimension sweep further shows that increasing $r_c$ from 2 to 8
improves mean performance, while larger codes yield no consistent additional gains across metrics (Appendix~\ref{app:code_dimension_sensitivity}).

\section{Conclusion and Future Work}

We revisited personalized adaptation through the lens of personalization capacity allocation and conducted three empirical analyses. Our findings motivated LINEUP, which directly learns fine-grained reusable factors across users, composes them using user history and the current query, and confines user-specific fitting to a compact residual correction. Our theoretical analysis bounds the risk of finite-step code refinement under limited history. Across six LaMP tasks, LINEUP achieves the best mean performance while optimizing only eight scalars per target user, with clear advantages under limited history and low incremental parameter and storage costs. These results show that rich personalization need not require a full private adapter for every user. Shared learning and conditional composition provide rich personalization capacity, while independent user adaptation is reserved for compact individual correction. Future work will extend this task-isolated formulation to cross-task personalization, capturing stable user preferences through reusable factors.

\subsection*{AI use statement}

In this work, we used generative AI tools, including ChatGPT (\url{https://chatgpt.com/}) and Claude (\url{https://claude.ai/}), to
assist with polishing the writing. Specifically, we used these tools to polish text, check grammar, and correct formatting errors throughout the writing process. We have carefully reviewed all AI-assisted work to ensure there is no falsehood, plagiarism, or misrepresentation. We take responsibility for the final content of this work, including any text, claims or artifacts produced with the aid of generative AI.

\subsection*{Ethics statement}

Personalized large language models raise privacy, bias, and resource concerns because they rely on user-specific histories and adaptation states. First, LINEUP replaces complete private adapters with eight-dimensional user
codes, substantially reducing the persistent user-specific model state that must be stored and protected. This compact design can simplify personalized-parameter management. User histories, codes, and shared components still require appropriate access controls and privacy safeguards. Second, history-based personalization may reflect biases or manipulation in user data. LINEUP's separation of shared components and individual codes offers a basis for targeted auditing. Output review and transparent user controls can help address undesirable personalization and support responsible use. Third, LLM training and inference entail computational and environmental costs. By reusing shared personalization capacity rather than maintaining a
complete adapter for every user, LINEUP substantially reduces per-user parameter and storage requirements, supporting more resource-efficient personalization as the user population grows.

\bibliography{iclr2027_conference}
\bibliographystyle{iclr2027_conference}


\clearpage
\appendix

\addtocontents{toc}{\protect\setcounter{tocdepth}{2}}

\renewcommand{\contentsname}{Appendix Contents}
\tableofcontents
\thispagestyle{fancy}

\clearpage

\input{Appendix/A_Related_Work}

\input{Appendix/B_Additional_Empirical_Analysis}

\input{Appendix/C_Methodology_Details}

\input{Appendix/D_Theoretical_Analysis_and_Proofs}

\input{Appendix/E_Supplementary_Materials_for_Experimental_Setup}

\input{Appendix/F_Supplementary_Experimental_Results}

\end{document}

%% file: math_commands.tex
\usepackage{amsmath,amsfonts,bm}

\def\eqref#1{equation~\ref{#1}}

\def\1{\bm{1}}

\DeclareMathAlphabet{\mathsfit}{\encodingdefault}{\sfdefault}{m}{sl}
\SetMathAlphabet{\mathsfit}{bold}{\encodingdefault}{\sfdefault}{bx}{n}



%% file: Appendix/A_Related_Work.tex
\section{Related Work}
\label{app:related_work}

\subsection{Personalized Large Language Models}
\label{app:related_work_pllm}

Large language models (LLMs) have demonstrated strong capabilities across language understanding and generation \citep{10.1145/3735632, chen-etal-2024-reconcile, NEURIPS2024_0356216f, sui-etal-2025-chain, le-cong-etal-2025-llms, 10.1145/3774904.3792376}. Personalized large language models adapt model behavior to individual users based on their histories, profiles, or preference feedback \citep{liu2025surveypersonalizedlargelanguage}. Existing methods can be broadly grouped into tuning-free, learned adaptation-based, and alignment approaches.

\noindent\textbf{Tuning-free personalization.} Tuning-free methods typically summarize or retrieve user histories to construct personalized prompts or augment the model context
\citep{richardson2023integratingsummarizationretrievalenhanced, qiu-etal-2025-measuring, 10.1145/3626772.3657783, zerhoudi2026personaragenhancingretrievalaugmentedgeneration}. Other approaches steer generation at inference time by contrasting personalized and non-personalized model states or combining decoding signals \citep{zhang-etal-2025-personalized, lv2026costeercollaborativedecodingtimepersonalization}. These methods avoid target-user parameter optimization and can readily incorporate newly available histories. However, their effectiveness depends on the quality of historical data preprocessing.

\noindent\textbf{Learned adaptation-based personalization.}
This is a complementary family that encodes personalization into learned parameters or representation-space interventions. We organize them according to the basic adaptation unit from which target-user personalization is constructed.

The first direction performs \textbf{individualized adaptation construction}. OPPU directly optimizes an independent PEFT module for each user by using their history \citep{tan-etal-2024-democratizing}. PerFit decomposes representation-space personalization into collective and personalized shifts while retaining a user-specific intervention for individual variation \citep{ICLR2026_7a1ac10e}. P2P changes how such individualized adaptation is obtained by introducing a population-trained hypernetwork. It maps a user profile gained during preprocessing and the user's query-retrieved history to a complete set of personalized LoRA parameters. P2P thereby avoids iterative per-user fine-tuning at deployment \citep{tan-etal-2026-instant}. Despite these differences, the primary personalization object remains an individualized adaptation module that can be optimized directly, instantiated in representation space, or generated from user information. Methods that fit and retain such modules incur user-specific optimization and persistent-state costs, while hypernetwork generation shifts part of this burden to preprocessing stages. Thus, these approaches mainly improve how an individualized adaptation state is obtained, rather than changing the granularity at which personalization capacity itself is organized.

The second direction improves reuse through \textbf{user-derived reusable components}. Per-Pcs learns personalized PEFT modules for selected users, decomposes them into reusable pieces, and assembles pieces for target users \citep{tan-etal-2024-personalized}. PROPER introduces intermediate group-level adaptation and combines group LoRAs with individual user LoRAs \citep{zhang-etal-2025-proper}. MTA constructs a bank of anchor-derived Meta-LoRAs, retrieves and merges relevant modules for a target user, and additionally fits an individual LoRA for few-shot refinement \citep{li-etal-2026-mta}. These approaches demonstrate the value of cross-user reuse. Nevertheless, their reusable units are obtained from adaptation modules associated with particular users. Their emphasis is on reusing such modules, rather than systematically characterizing what structure is reusable. This leaves open a different organization of personalization capacity: whether reusable units can instead be learned \emph{directly} as cross-user personalization factors and composed according to the current request.

\noindent\textbf{Personalized alignment.} Personalized alignment forms a complementary line of research that typically models personalization through heterogeneous objectives or preference signals. User preferences can be incorporated through personalized reward modeling, parameter merging, or decoding-time policy combination \citep{liu2025surveypersonalizedlargelanguage, NEURIPS2023_e12a3b98, NEURIPS2024_57c89126, NEURIPS2025_2fbc1ea6, ICLR2026_c0ae4874}. These methods usually assume explicit preference feedback or a predefined preference space. By contrast, our setting focuses on task-specific personalization inferred from heterogeneous behavioral histories, where user-specific behavior must be recovered from prior interactions rather than supplied as an explicit alignment objective.

\noindent\textbf{Positioning of our work.} We take \emph{personalization capacity allocation} as the organizing principle. The central distinction is not whether parameters are shared across users, but \emph{what is shared, how reusable capacity is learned, and what must still be independently adapted for a target user}.

LINEUP differs from individualized adaptation methods by changing the basic personalization unit. Rather than constructing a complete adaptation state for each user or request, it directly learns fine-grained low-rank personalization factors across source users as reusable population-level capacity. Different users can reuse overlapping factors, while the same user can employ different compositions for different requests. LINEUP also differs from collaborative adapter-reuse methods in the origin of this reusable capacity, since its factor bank is learned jointly for cross-user composition from the outset, rather than being obtained by first learning user-, anchor-, or group-associated adapters and subsequently decomposing, selecting, or merging them.

Our composition mechanism further separates persistent user relevance from request-specific utility. A user's history first recalls a compact candidate set of factors, and the current query then calibrates their sparse selection and mixture. LINEUP further confines target-user optimization to an eight-dimensional code over shared correction maps. Therefore, most personalization capacity resides in reusable shared structures, while only a tiny coordinate state must be independently estimated for a new target user. This factor-centric design is motivated by our empirical findings.

Taken together, LINEUP \emph{decouples personalization capacity from per-user adaptation}: expressive personalization is learned predominantly through reusable, conditionally composed shared capacity, whereas independent target-user fitting is reserved for a compact residual correction. This organization simultaneously targets strong personalization, adaptation from limited user history, and low incremental parameter and storage cost as the user population grows.

\subsection{Mixture of Experts}
\label{app:related_work_moe}

Mixture-of-Experts (MoE) models combine specialized expert modules through learned routing mechanisms \citep{6797059, 10937907}. Sparse MoE activates only a subset of experts for each input, enabling specialized and compositional model capacity without proportional computation \citep{shazeer2017outrageously, pmlr-v162-du22c, dai-etal-2024-deepseekmoe}. Beyond general-purpose model scaling, MoE-style architectures have also been studied in decentralized collaborative training and personalized federated learning \citep{liu-etal-2026-modularmoe, zhang2026pc, 11359536, 10.1145/3696410.3714561}. In much of this literature, personalization is defined at the client level, with shared or private experts introduced to accommodate heterogeneity in local data distributions, model architectures, or computational resources \citep{9743558, 10.1016/j.eswa.2023.122874}. Personalized LLMs focus on a different problem where the model must infer latent individual preferences from behavioral histories and tailor its behavior jointly to the user and the current request.

LINEUP adopts the compositional principle of sparse expert models but uses it for a different granularity and purpose. Its experts are fine-grained low-rank personalization factors jointly learned as reusable cross-user capacity, rather than independently meaningful user or client models. No single factor is intended to represent a complete user. Instead, different users can reuse overlapping factors, and personalization emerges from their conditional composition. Routing is explicitly hierarchical, where user history first narrows the shared factor bank to a candidate set, after which the current query calibrates candidate factors. This separates persistent user relevance from request-specific utility rather than collapsing both into a single routing decision. LINEUP therefore uses expert-style composition as part of a broader personalization capacity allocation strategy, rather than as a conventional mechanism for model scaling or client specialization.

%% file: Appendix/B_Additional_Empirical_Analysis.tex
\setcounter{secnumdepth}{3}

\section{Additional Empirical Analysis}
\label{app:analysis}

This appendix provides definitions and supporting evidence for Section~\ref{sec:analysis}. The numbered subsections below provide direct access to each diagnostic definition, comparison, and validation.

\noindent\textbf{RQ1: Cross-user reuse.}
Shared subspaces and update coverage (Appendix~\ref{app:rq1_setup});
task-adapter control (Appendix~\ref{app:rq1_task_control}).

\noindent\textbf{RQ2: Direction utility.}
Setup and utility definition (Appendix~\ref{app:rq2_setup}); within-user
and cross-user agreement (Appendix~\ref{app:rq2_agreement});
scale-controlled query variation (Appendix~\ref{app:rq2_variation});
direction-removal validation (Appendix~\ref{app:rq2_removal}).

\noindent\textbf{RQ3: History-based correction.}
Correction signals and label controls (Appendix~\ref{app:rq3_setup});
support-to-query transfer (Appendix~\ref{app:rq3_transfer}); compact
correction spaces (Appendix~\ref{app:rq3_compact_metrics}); first-order
intervention validation (Appendix~\ref{app:rq3_validation}).

\paragraph{Diagnostic reference.}
RQ1 and RQ2 use independently fitted OPPU user LoRAs
\citep{tan-etal-2024-democratizing} on a common task-adapted backbone.
We analyze their effective updates $D_{u,i}=s_iB_{u,i}A_{u,i}$ rather than
the individual LoRA factors. RQ3 probes representation corrections in the
task-adapted backbone without user-specific adapters. Diagnostic user folds
and learned bases are separate from LINEUP's shared-training procedure and
are not used to train LINEUP.

\paragraph{Uncertainty and aggregation.}
We use fixed checkpoints. Unless stated otherwise, intervals are
pointwise 95\% percentile confidence intervals from 2,000 user-level
bootstrap draws, with checkpoints, diagnostic folds, and learned bases held
fixed. These intervals quantify user-sampling uncertainty rather than
variation across training seeds. Aggregation is specified below; RQ2 also
accounts for dependence between user pairs.

\subsection{RQ1: Cross-User Reuse}
\label{app:analysis_sharing}

\subsubsection{Shared Subspaces and Update Coverage}
\label{app:rq1_setup}

\paragraph{Diagnostic setup.}
We study LaMP-2M, LaMP-2N, LaMP-5, and LaMP-7, covering classification
and generation. Each user contributes 64 adaptation sites: the query and
value projections in all 32 layers. Every site has output dimension 4096.
Evaluated users do not participate in learning the shared directions.

\paragraph{Shared-subspace construction.}
For diagnostic fold $f$, let $\mathcal V_{-f}$ denote the users outside
the held-out fold. At site $i$, we construct the normalized output
covariance
\begin{equation}
C_i^{(-f)}
=
\frac{1}{|\mathcal V_{-f}|}
\sum_{u\in\mathcal V_{-f}}
\frac{D_{u,i}D_{u,i}^{\top}}{\|D_{u,i}\|_F^2}.
\label{eq:app_shared_covariance}
\end{equation}
The columns of $Q_{m,i}^{(-f)}$ are its leading $m$ eigenvectors.
Normalization gives users equal weight regardless of update magnitude.
Each held-out user is evaluated using a basis learned outside that user's
fold.

\paragraph{Coverage and random reference.}
For a held-out user, coverage measures the fraction of update energy
captured by the shared output subspace:
\begin{equation}
\operatorname{Coverage}_{u,i}(m)
=
\frac{\|(Q_{m,i}^{(-f)})^{\top}D_{u,i}\|_F^2}
     {\|D_{u,i}\|_F^2}.
\label{eq:app_update_coverage}
\end{equation}
This measures reuse of output directions while allowing their
coefficients to differ across users.
We average over held-out user--site pairs and evaluate
$m\in\{1,2,4,8,16,32,64,128\}$.
The random reference in Figure~\ref{fig:main_analysis}(a) averages five
Haar-random subspace draws, with coverage approximately 0.0078 at $m=32$
and 0.0312 at $m=128$. Coverage at $m=32$ is also reported in
Table~\ref{tab:random_partition_control} under genuine-user basis.

\subsubsection{Random-Partition Control}
\label{app:rq1_random_partition_control}

\paragraph{Protocol.}
We retain the diagnostic cohorts of 100, 49, 101, and 100 users on
LaMP-2M/2N/5/7, respectively.
Within each diagnostic fold, we pool the supervised examples used to
fit the genuine-user adapters and randomly repartition them into
pseudo-users, preserving every input--output pair and the original
group-size sequence. Each pseudo-user adapter is independently fitted
using the same task-adapted backbone, LoRA configuration, and
optimization schedule as its genuine-user counterpart.
We use one fixed random partition per task.
Partitioning within folds keeps held-out users' examples out of
basis construction.

\paragraph{Matched coverage comparison.}
We construct each basis from either genuine-user or random-partition
adapters outside the held-out fold, using
Equation~\ref{eq:app_shared_covariance}.
Both bases are evaluated on the same held-out genuine-user updates
using Equation~\ref{eq:app_update_coverage}.
We denote the genuine-minus-random coverage difference by
$\Delta_{\mathrm{group}}$.
Its 95\% confidence interval uses 2,000 paired user-level bootstrap
draws after averaging sitewise differences within each user.
Checkpoints, diagnostic folds, learned bases, and the random
partition remain fixed.

\paragraph{Results.}
At $m=32$, genuine-user bases provide higher coverage on all four
tasks, with absolute advantages of 0.143--0.210 and paired confidence
intervals excluding zero (Table~\ref{tab:random_partition_control});
reported differences are computed from unrounded coverage values.
Removing the task-adapter output span from both adapter populations
and repeating basis construction and evaluation
(Appendix~\ref{app:rq1_task_control}) leaves these gaps nearly unchanged.
Thus, preserving genuine user grouping yields shared directions that
better capture held-out genuine-user updates, strengthening the
cross-user reuse evidence in
\hyperref[obs:sharing]{Observation~1}.

\begin{table}[t]
\setlength{\tabcolsep}{4.0pt}
\renewcommand{\arraystretch}{1.12}

\caption{
Held-out update-energy coverage at $m=32$ for genuine-user and
size-matched random-partition bases. Genuine-user grouping improves coverage, including after task-span removal.
}
\label{tab:random_partition_control}
\begin{center}
\begin{tabular}{@{}lccccc@{}}
\toprule
\textbf{Task}
& \shortstack{\textbf{Genuine-user}\\\textbf{basis}}
& \shortstack{\textbf{Random-partition}\\\textbf{basis}}
& $\boldsymbol{\Delta}_{\mathrm{group}}$
& \textbf{95\% CI}
& \shortstack{\textbf{Post-removal}\\$\boldsymbol{\Delta}_{\mathrm{group}}$} \\
\midrule

LaMP-2M
& $.541$
& $.377$
& $\textcolor{tablegreen}{\mathbf{+.164}}$
& $[.159,\ .171]$
& $+.164$ \\

LaMP-2N
& $.621$
& $.411$
& $\textcolor{tablegreen}{\mathbf{+.210}}$
& $[.195,\ .226]$
& $+.212$ \\

LaMP-5
& $.679$
& $.536$
& $\textcolor{tablegreen}{\mathbf{+.143}}$
& $[.141,\ .147]$
& $+.144$ \\

LaMP-7
& $.620$
& $.462$
& $\textcolor{tablegreen}{\mathbf{+.158}}$
& $[.155,\ .163]$
& $+.158$ \\

\bottomrule
\end{tabular}
\end{center}
\end{table}

\subsubsection{Task-Adapter Control}
\label{app:rq1_task_control}

To test whether the observed coverage is concentrated in the trained
task adapter's output span, let $Q_{\mathrm{task},i}$ be an orthonormal
basis of the task-LoRA output span.
Its overlap with a user update is
$\|Q_{\mathrm{task},i}^{\top}D_{u,i}\|_F^2/\|D_{u,i}\|_F^2$.
Mean overlap is 2.85\%, 3.68\%, 3.65\%, and 2.00\% on
LaMP-2M/2N/5/7, respectively
(Figure~\ref{fig:appendix_analysis_1}(a)).

We remove this span,
\begin{equation}
D_{u,i}^{\perp}
=
\bigl(I-Q_{\mathrm{task},i}Q_{\mathrm{task},i}^{\top}\bigr)D_{u,i},
\label{eq:app_task_span_removal}
\end{equation}
and repeat shared-basis construction and coverage evaluation on the
residual updates. At $m=32$, residual coverage is 0.533, 0.613, 0.672,
and 0.615 on LaMP-2M/2N/5/7, respectively
(Figure~\ref{fig:appendix_analysis_1}(b)).
Coverage remains nearly unchanged, supporting the reusable structure in
\hyperref[obs:sharing]{Observation~1} beyond the task adapter's output span.

\begin{figure}[t]
    \centering
    \begin{minipage}[t]{0.48\linewidth}
        \centering
        \includegraphics[width=\linewidth]{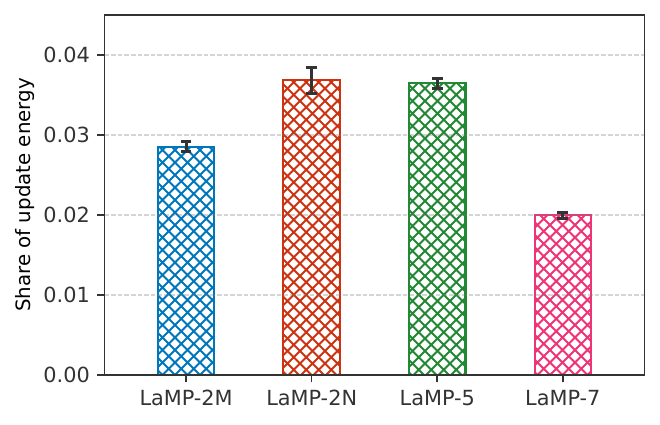}\\[-1mm]
        (a) Task-span overlap
    \end{minipage}\hfill%
    \begin{minipage}[t]{0.48\linewidth}
        \centering
        \includegraphics[width=\linewidth]{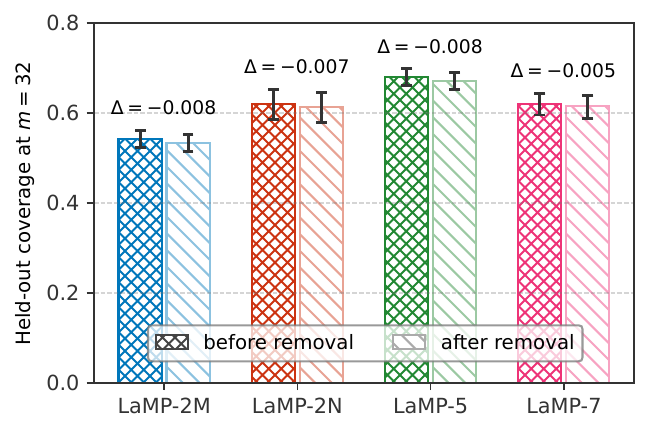}\\[-1mm]
        (b) Task-span removal
    \end{minipage}

    \caption{Task-adapter controls for
    \hyperref[sec:analysis_sharing]{RQ~1}.
    (a) Only a small fraction of user-update energy lies in the task-LoRA
    output span.
    (b) Removing this span leaves shared-subspace coverage nearly unchanged.}
    \label{fig:appendix_analysis_1}
\end{figure}

\subsection{RQ2: User- and Query-Dependent Utility}
\label{app:analysis_composing}

\subsubsection{Diagnostic Setup and Direction Utility}
\label{app:rq2_setup}

\paragraph{Diagnostic setup.}
We use LaMP-2M and LaMP-2N because they provide repeated held-out queries
per user, allowing utility to be compared while the fitted adapter remains
fixed. The diagnostic uses 100 users with 3,302 queries on LaMP-2M and
49 users with 6,033 queries on LaMP-2N. At each of the 64 sites, we use
the first eight shared directions learned outside the evaluated user's
fold.

\paragraph{Direction-utility definition.}
Let $q_{i,k}$ denote shared direction $k$ at site $i$. Its component in
user $u$'s fixed update is
$E_{u,i,k}=q_{i,k}q_{i,k}^{\top}D_{u,i}$.
We introduce diagnostic gates through
\begin{equation}
D_{u,i}(\boldsymbol\alpha_i)
=
D_{u,i}
+
\sum_{k=1}^{8}(\alpha_{i,k}-1)E_{u,i,k}.
\label{eq:app_gated_update}
\end{equation}
All gates equal to one recover the original adapter. Let
$\mathcal L_{u,j}(\boldsymbol\alpha)$ denote the query loss under the
gated model, where $\boldsymbol\alpha$ collects gates across sites. The
signed sensitivity is
\begin{equation}
G_{u,j,i,k}
=
\left.
\frac{\partial\mathcal L_{u,j}(\boldsymbol\alpha)}
     {\partial\alpha_{i,k}}
\right|_{\boldsymbol\alpha=\mathbf 1}.
\label{eq:app_gate_sensitivity}
\end{equation}
The sign indicates whether strengthening a component raises or
lowers loss to first order. The user adapter and diagnostic directions remain fixed
across queries; the gates are diagnostic variables, not a trained router.

\subsubsection{Within-User and Cross-User Agreement}
\label{app:rq2_agreement}

For each user, we use the first eight held-out queries. Within-user
agreement averages Spearman correlation over all $\binom{8}{2}=28$ query
pairs, then over users and sites. For cross-user agreement, let
$z_{u,j,i}$ denote the centered, unit-length rank vector of the eight
signed sensitivities, and define
$\bar z_{u,i}=\frac{1}{8}\sum_{j=1}^{8}z_{u,j,i}$.
The mean query-pair agreement between distinct users $u$ and $v$ is
\begin{equation}
\frac{1}{8^2}
\sum_{j=1}^{8}\sum_{j'=1}^{8}
z_{u,j,i}^{\top}z_{v,j',i}
=
\bar z_{u,i}^{\top}\bar z_{v,i}.
\label{eq:app_cross_user_agreement}
\end{equation}
We average these values over distinct user pairs and sites. Within-user
intervals resample users as clusters; cross-user intervals use a
user-vertex bootstrap to preserve dependence between pairs sharing a user.

Within- versus cross-user agreement is
$0.145\,[0.107,0.185]$ versus $0.008\,[0.004,0.014]$ on LaMP-2M and
$0.400\,[0.330,0.475]$ versus $0.059\,[0.021,0.112]$ on LaMP-2N
(Figure~\ref{fig:main_analysis}(b)). The within-user value is higher at all
64 sites on both tasks (Figure~\ref{fig:appendix_analysis_2}(b)).

\subsubsection{Scale-Controlled Query Variation}
\label{app:rq2_variation}

To distinguish changes in relative directional responses from overall
sensitivity scale, we normalize each eight-dimensional sensitivity vector
as $X_{u,j,i}=G_{u,j,i}/\|G_{u,j,i}\|_2$.
Let $SS_{\mathrm W,i}$ and $SS_{\mathrm B,i}$ denote the within-user and
between-user sums of squares of these normalized responses at site $i$.
We report
\begin{equation}
V_{\mathrm{query}}
=
\frac{\sum_{i\in\mathcal I}SS_{\mathrm W,i}}
     {\sum_{i\in\mathcal I}(SS_{\mathrm W,i}+SS_{\mathrm B,i})}.
\label{eq:app_query_variation}
\end{equation}
This fraction is $0.827\,[0.778,0.869]$ on LaMP-2M and
$0.613\,[0.527,0.699]$ on LaMP-2N
(Figure~\ref{fig:appendix_analysis_2}(a)). Substantial query-dependent
variation therefore remains after controlling for overall sensitivity
scale.

\subsubsection{Exact Direction-Removal Validation}
\label{app:rq2_removal}

We compare the first-order removal prediction $-G_{u,j,i,k}$ with the
exact query-loss change when one gate is set from one to zero while all
other gates remain unchanged. Each task contributes 512 interventions:
two users, eight queries per user, eight directions, and four sites.
Pooled Pearson correlations are 0.998955 on LaMP-2M and 0.999295 on
LaMP-2N (Figure~\ref{fig:appendix_analysis_2}(c)). This agreement
supports the functional interpretation of the signed sensitivities in
\hyperref[obs:composing]{Observation~2}.

\begin{figure}[t]
    \centering
    \begin{minipage}[t]{0.26\linewidth}
        \centering
        \includegraphics[height=3.6cm]{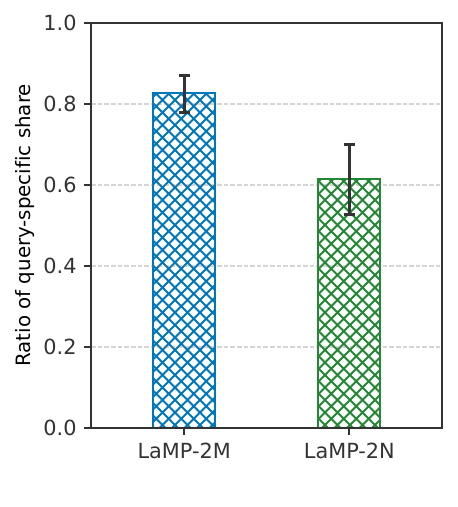}\\[-1mm]
        (a) Query-specific variation
    \end{minipage}\hfill%
    \begin{minipage}[t]{0.43\linewidth}
        \centering
        \includegraphics[height=3.6cm]{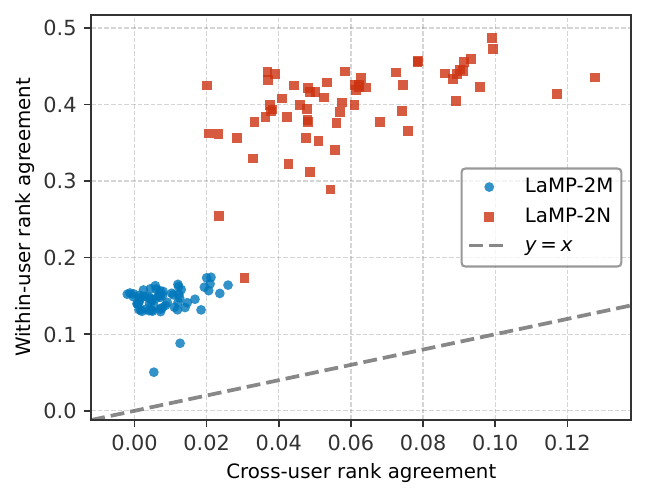}\\[-1mm]
        (b) Sitewise rank agreement
    \end{minipage}\hfill%
    \begin{minipage}[t]{0.28\linewidth}
        \centering
        \includegraphics[height=3.6cm]{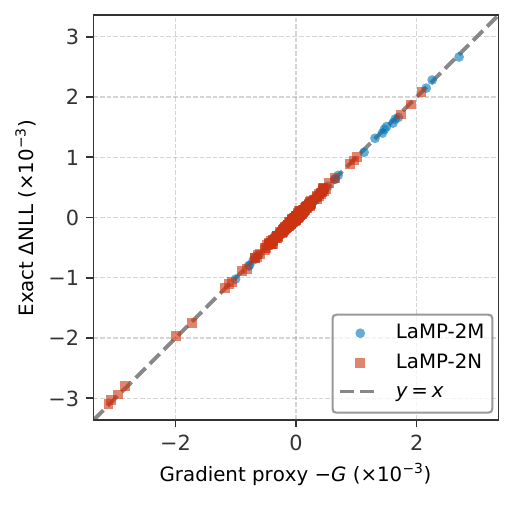}\\[-1mm]
        (c) Direction-removal validation
    \end{minipage}

    \caption{Additional evidence for user- and query-dependent utility in
    \hyperref[sec:analysis_composing]{RQ~2}.
    (a) Substantial query variation remains after sensitivity-scale control.
    (b) Within-user rank agreement exceeds cross-user agreement at every
    analyzed site.
    (c) First-order predictions closely match exact direction-removal loss
    changes.}
    \label{fig:appendix_analysis_2}
\end{figure}

\subsection{RQ3: Compact History-Based Correction}
\label{app:analysis_correcting}

\subsubsection{Correction Signals and Label Controls}
\label{app:rq3_setup}

\paragraph{Diagnostic setup and correction signals.}
We use LaMP-2M and LaMP-2N, whose finite categorical label spaces allow
label-based controls. The task-adapted backbone is frozen and contains no
user-specific adapters. At zero-indexed layers
$\mathcal L_{\mathrm{diag}}=\{0,4,9,13,18,22,27,31\}$, we add a vector
$v_\ell\in\mathbb R^{4096}$ to all $T_z$ non-padding token positions of
example $z$. Let $\bar{\mathcal L}_{\mathcal T}(z;v_\ell)$ denote the
mean target-token NLL under this intervention. Following the
representation-space view in Section~\ref{sec:lora_preliminaries}, the
length-rescaled correction signal is
\begin{equation}
\gamma_z^\ell
=
-\frac{1}{T_z}
\left.
\nabla_{v_\ell}\bar{\mathcal L}_{\mathcal T}(z;v_\ell)
\right|_{v_\ell=0}.
\label{eq:app_correction_signal}
\end{equation}
This rescales the loss-reducing gradient by the number of positions
receiving the intervention; it is not unit-normalized.

Diagnostic support $\mathcal S_u^{\mathrm{diag}}\subseteq\mathcal H_u$
contains the first $\min(32,|\mathcal H_u|)$ available history records.
Their mean correction direction is
\begin{equation}
\bar s_u^\ell
=
\frac{1}{|\mathcal S_u^{\mathrm{diag}}|}
\sum_{z\in\mathcal S_u^{\mathrm{diag}}}\gamma_z^\ell.
\label{eq:app_support_direction_raw}
\end{equation}

\paragraph{Label controls and donor histories.}
For each diagnostic fold and layer, we compute the record-weighted mean
correction signal $m_c^\ell$ for each label $c$ using only diagnostic
support from analysis-training users. Let $Q_{\mathrm{label}}^\ell$ be
an orthonormal basis of
$\operatorname{span}\{m_1^\ell,\ldots,m_C^\ell\}$.
Both tasks have $C=15$ labels, and the span has dimension 15 at every
analyzed fold and layer. We remove this span from query and support signals:
\begin{equation}
\begin{aligned}
\gamma_z^{\ell,\perp}
&=
\bigl(I-Q_{\mathrm{label}}^\ell(Q_{\mathrm{label}}^\ell)^\top\bigr)
\gamma_z^\ell,\\
s_u^{\ell,\perp}
&=
\bigl(I-Q_{\mathrm{label}}^\ell(Q_{\mathrm{label}}^\ell)^\top\bigr)
\bar s_u^\ell.
\end{aligned}
\label{eq:app_label_projection}
\end{equation}
Held-out query signals are used only for evaluation, not to construct the
label span.

We compare three history sources. \emph{Own} uses the target user's
history. \emph{Random} uses a randomly selected donor from another
diagnostic fold. \emph{Matched} uses an out-of-fold donor with the nearest
historical label distribution in $L_1$ distance. All support directions
follow the same construction and label-span removal.

\subsubsection{Support-to-Query Transfer}
\label{app:rq3_transfer}

For history source $a$, let $s_{a(u)}^{\ell,\perp}$ denote its support
direction. Signed transfer to held-out query $j$ of user $u$ is
\begin{equation}
T_{u,j,\ell}^{(a)}
=
\bigl(\gamma_{u,j}^{\ell,\perp}\bigr)^\top s_{a(u)}^{\ell,\perp}.
\label{eq:app_full_transfer}
\end{equation}
A positive value indicates alignment between the history-derived
correction and a first-order loss-reducing query direction. The statistic
retains signal magnitudes; token count and intervention step enter the
corresponding first-order NLL prediction
(Equation~\ref{eq:app_first_order_injection}). We average over held-out
query records within each layer, then average the eight layerwise means
equally.

Own-history transfer is $2.190\times10^{-6}$ with interval
$[0.798,3.906]\times10^{-6}$ on LaMP-2M and $1.930\times10^{-5}$ with
interval $[0.808,3.058]\times10^{-5}$ on LaMP-2N
(Figure~\ref{fig:main_analysis}(c)). Own-history point estimates are
positive at all eight layers; intervals exclude zero at seven layers on
LaMP-2M and all eight on LaMP-2N. Random and matched donor intervals
cross zero on LaMP-2M, while both donor means are negative on LaMP-2N.
This consistent own-history alignment supports individual correction
beyond the removed label-mean span
(\hyperref[obs:correcting]{Observation~3}).

\subsubsection{Compact Correction Spaces}
\label{app:rq3_compact_metrics}

For each layer and diagnostic fold, we learn an orthonormal correction
basis $R_{r,\ell}$ using only support-side signals from analysis-training
users outside the evaluated fold. Held-out query signals are used only
for evaluation. Let $P_{r,\ell}=R_{r,\ell}R_{r,\ell}^{\top}$.
We measure both retained correction energy and retained support-to-query
transfer. Residual query-gradient coverage is
\begin{equation}
\operatorname{Coverage}^{\mathrm{corr}}_{u,j,\ell}(r)
=
\frac{\|R_{r,\ell}^{\top}\gamma_{u,j}^{\ell,\perp}\|_2^2}
     {\|\gamma_{u,j}^{\ell,\perp}\|_2^2}.
\label{eq:app_correction_coverage}
\end{equation}
Writing $T_{u,j,\ell}^{\mathrm{own}}$ for the full-space transfer under
\emph{Own}, its projected counterpart is
\begin{equation}
T_{u,j,\ell}^{\mathrm{own}}(r)
=
\bigl(\gamma_{u,j}^{\ell,\perp}\bigr)^\top
P_{r,\ell}s_u^{\ell,\perp}.
\label{eq:app_projected_transfer}
\end{equation}
For the set $\mathcal J$ of held-out user--query records, transfer
retention is
\begin{equation}
\operatorname{Retention}(r)
=
\frac{1}{|\mathcal L_{\mathrm{diag}}|}
\sum_{\ell\in\mathcal L_{\mathrm{diag}}}
\frac{\sum_{(u,j)\in\mathcal J}T_{u,j,\ell}^{\mathrm{own}}(r)}
     {\sum_{(u,j)\in\mathcal J}T_{u,j,\ell}^{\mathrm{own}}}.
\label{eq:app_transfer_retention}
\end{equation}
The transfer ratio is computed within each layer before averaging equally
across layers. Coverage and retention use different references: residual
query-gradient energy and full-space, label-controlled support-to-query
transfer, respectively.

At $r=8$, energy coverage is $0.129\,[0.123,0.136]$ on LaMP-2M and
$0.170\,[0.160,0.181]$ on LaMP-2N, whereas transfer retention is
$0.347\,[0.222,0.479]$ and $0.294\,[0.198,0.413]$, respectively
(Figure~\ref{fig:appendix_analysis_3}(a)).
Figures~\ref{fig:appendix_analysis_3}(b,c) show the corresponding trends
over $r\in\{1,2,4,8,16\}$. At $r=16$, coverage reaches 0.168 and 0.244,
while retention reaches 0.425 and 0.416. These comparisons support the
preferential retention of transferable correction in
\hyperref[obs:correcting]{Observation~3}. The diagnostic sweep is not
used to select LINEUP's user-code dimension.

\subsubsection{First-Order Intervention Validation}
\label{app:rq3_validation}

We compare first-order predictions with exact loss changes after injecting
the own-history direction projected into the $r=8$ diagnostic space.
Because Equation~\ref{eq:app_correction_signal} includes $1/T_z$, the
predicted change in mean target-token NLL is
\begin{equation}
\Delta\mathrm{NLL}_z
\approx
-T_z\,(\gamma_z^\ell)^\top v,
\label{eq:app_first_order_injection}
\end{equation}
where $v$ is the projected support direction scaled by a small step size
calibrated on analysis-training users.

Each task contributes 64 interventions: two users, four queries per user,
and eight layers. Median layerwise Pearson correlations are 0.9998 on
LaMP-2M and 0.9997 on LaMP-2N
(Figure~\ref{fig:appendix_analysis_3}(d)).

\begin{figure}[t]
    \centering
    \begin{minipage}[t]{0.245\linewidth}
        \centering
        \includegraphics[width=\linewidth]{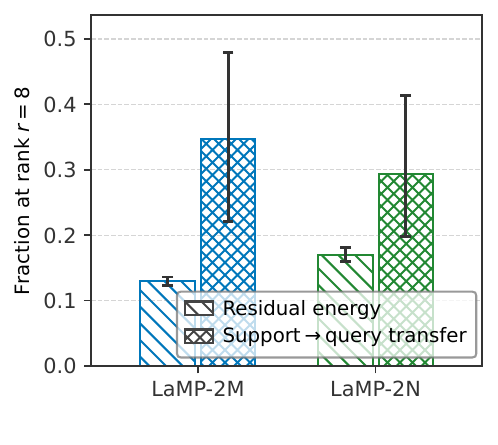}\\[-1mm]
        (a) Transfer vs.\ energy
    \end{minipage}\hfill%
    \begin{minipage}[t]{0.245\linewidth}
        \centering
        \includegraphics[width=\linewidth]{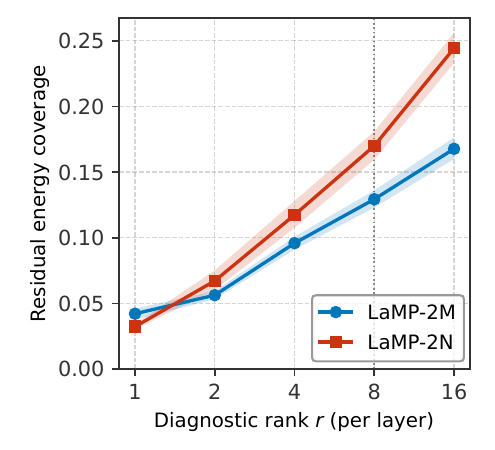}\\[-1mm]
        (b) Residual-energy coverage
    \end{minipage}\hfill%
    \begin{minipage}[t]{0.245\linewidth}
        \centering
        \includegraphics[width=\linewidth]{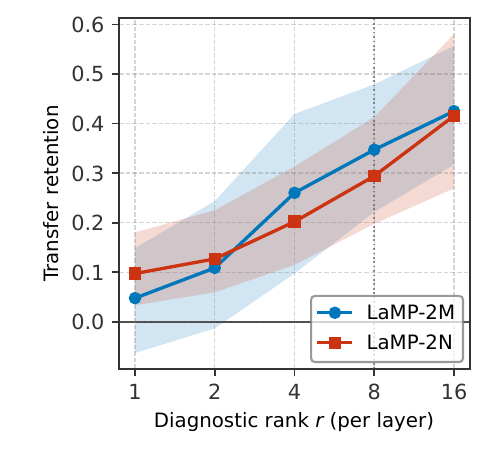}\\[-1mm]
        (c) Transfer retention
    \end{minipage}\hfill%
    \begin{minipage}[t]{0.245\linewidth}
        \centering
        \includegraphics[width=\linewidth]{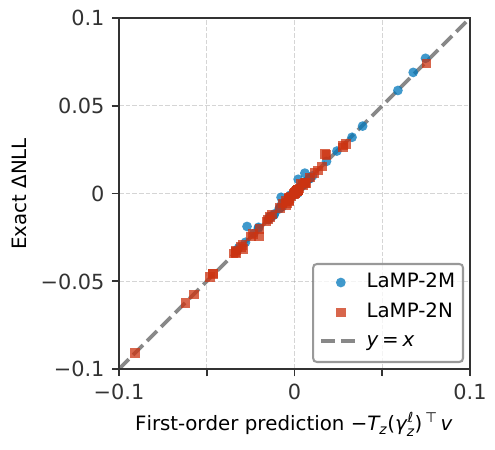}\\[-1mm]
        (d) First-order intervention validation
    \end{minipage}

    \caption{Additional evidence for compact history-based correction in
    \hyperref[sec:analysis_correcting]{RQ~3}.
    (a) At rank eight, transfer retention exceeds residual-energy coverage.
    (b) Residual-energy coverage increases with diagnostic rank.
    (c) Higher ranks retain more support-to-query transfer.
    (d) First-order predictions track exact intervention loss changes in most
    tested cases.}
    \label{fig:appendix_analysis_3}
\end{figure}

%% file: Appendix/C_Methodology_Details.tex
\section{Methodology Details}
\label{app:method_details}

This appendix defines the shared factors, conditional composition, and
user-code adaptation in Section~\ref{sec:method}.
Algorithm~\ref{alg:lineup} summarizes shared training and target-user
personalization. Task-specific hyperparameters and model-selection
settings are reported with the experimental configuration.

\subsection{Algorithm of LINEUP}
\label{app:lineup_algorithm}

\paragraph{Shared parameters and user state.}
The S1 and S2 shared parameters are
\[
\phi_1=\{A_{k,i},B_{k,i}\}_{k,i}\cup\{W_H,W_Q\},
\qquad
\phi_2=\{V_\ell\}_{\ell\in\mathcal L_{\mathrm{corr}}}\cup\{G\}.
\]
Together, $\phi_{\mathcal T}=\phi_1\cup\phi_2$ instantiates the shared
personalization parameters in
Section~\ref{sec:personalized_problem_setup}. The full method's only
independently adapted target-user state is
$\omega_u=\hat c_u\in\mathbb R^{r_{\mathrm c}}$.
The task-adapted backbone $\theta_{\mathcal T}$ and feature extraction
remain fixed throughout.

\paragraph{Training episodes and prediction loss.}
Each source-user episode contains a history $\mathcal H_u^{\mathrm{ep}}$
and supervised examples $\mathcal B_u$. Depending on the task, supervision
comes from history records held out from the episode history or from the
source user's task queries. For a predictor $f$ conditioned on the episode
history, define
\begin{equation}
\mathcal L_{\mathrm{tok}}(\mathcal B;f)
=
-\frac{1}{N(\mathcal B)}
\sum_{(x,y)\in\mathcal B}\sum_{t=1}^{|y|}
\log f(y_t\mid x,y_{<t}),
\label{eq:lineup_token_loss}
\end{equation}
where $N(\mathcal B)=\sum_{(x,y)\in\mathcal B}|y|$ counts supervised
output tokens, excluding prompt and padding tokens. Write $f_1$ for the
S1-only predictor and $f_2(c)$ for the predictor with S1 and code-controlled
correction. Both use the history encoding and each example's input for
routing.

\paragraph{Shared training.}
Stage 1 optimizes $\phi_1$ using $f_1$. Stage 2 freezes S1 and optimizes
$\phi_2$ using $f_2(G\tilde e_u)$ directly. The correction maps and
initializer are therefore learned from predictions at initialized codes,
before target-user support fitting. All shared parameters are then fixed
for target-user personalization.

\begin{algorithm}[t]
\caption{LINEUP}
\label{alg:lineup}
\begin{algorithmic}[1]
\Require Frozen task model $\theta_{\mathcal T}$, source users,
         target-user histories and query inputs, code-update budget $T$
\State Construct fixed feature maps and source-training centering statistics
\State Initialize $\phi_1$ and $\phi_2$

\Statex \textbf{Stage 1: Learn shared factor composition}
\For{each S1 training episode}
    \State Sample $(\mathcal H_u^{\mathrm{ep}},\mathcal B_u)$
           from a source user
    \State Encode the episode history; route each input as in
           Appendix~\ref{app:lineup_composition}
    \State Update $\phi_1$ on
           $\mathcal L_{\mathrm{tok}}(\mathcal B_u;f_1)$
\EndFor
\State Freeze $\phi_1$

\Statex \textbf{Stage 2: Learn shared corrections and initialization}
\For{each S2 training episode}
    \State Sample $(\mathcal H_u^{\mathrm{ep}},\mathcal B_u)$;
           compute $\tilde e_u$
    \State Set $c_u^{(0)}\gets G\tilde e_u$ and keep S1 active
    \State Update $\phi_2$ on
           $\mathcal L_{\mathrm{tok}}(\mathcal B_u;f_2(c_u^{(0)}))$
\EndFor
\State Freeze all shared parameters

\Statex \textbf{Target-user personalization and prediction}
\For{each target user $u$}
    \State Encode available $\mathcal H_u$; recall $\mathcal C_u$;
           construct $\mathcal S_u$
    \State Set $c_u^{(0)}\gets G\tilde e_u$,
           $c\gets c_u^{(0)}$; initialize fresh Adam state
    \For{$t=1,\ldots,T$}
        \State Route support inputs using the fixed history encoding;
               evaluate $\mathcal J_u(c)$
        \State Take one Adam update of $c$ only
    \EndFor
    \State Store $\hat c_u\gets c$
    \For{each query input $x$ of user $u$}
        \State Compute $\mathcal A_{u,x}$ and $w_{u,x}$ from
               the history encoding and current query
        \State Predict using the shared-factor mixture and correction
               from $\hat c_u$
    \EndFor
\EndFor
\end{algorithmic}
\end{algorithm}

\subsection{Shared Factors and Hierarchical Composition}
\label{app:lineup_composition}

\paragraph{Factor parameterization.}
At site $i$, factor $k$ contributes $F_{k,i}=s_iB_{k,i}A_{k,i}$, where
$A_{k,i}\in\mathbb R^{r_{\mathrm L}\times d_{\mathrm{in},i}}$,
$B_{k,i}\in\mathbb R^{d_{\mathrm{out},i}\times r_{\mathrm L}}$, and
$s_i$ is the LoRA scaling coefficient. Factor index $k$ identifies
coordinated low-rank updates across the query and value projections
throughout the transformer. These factors are learned directly through
source-user prediction under their conditional composition, rather than
obtained from the diagnostic adapters in Section~\ref{sec:analysis}.

\paragraph{History and query features.}
Let $E$ be the frozen input-embedding table and
$\bar e(s)=|s|^{-1}\sum_t E[s_t]$ the mean of the non-padding token
embeddings of text $s$. Historical records are rendered using their
available inputs and outputs. With separate fixed random projections
$P_H,P_Q\in\mathbb R^{d_f\times d}$, we compute
\begin{equation}
\begin{aligned}
e_u
&=P_H\frac{1}{|\mathcal H|}
  \sum_{h\in\mathcal H}\bar e(\operatorname{render}(h)),
&\tilde e_u&=\operatorname{norm}(e_u-\mu_H),\\
e_x
&=P_Q\bar e(\operatorname{prompt}(x)),
&\tilde e_x&=\operatorname{norm}(e_x-\mu_Q),
\end{aligned}
\label{eq:lineup_features}
\end{equation}
where $\operatorname{norm}(v)=v/\|v\|_2$. Here $\mathcal H$ is the
episode history during shared training and the available history during
target-user personalization. The centering vectors $\mu_H,\mu_Q$ are
estimated from source-training histories and query prompts and then fixed.
Each history record is pooled separately before averaging across records;
query encoding uses only the current input prompt. Historical outputs
provide personalization evidence, but evaluation targets never enter the
features. S1 recall and the S2 initializer use the same history encoding.

\paragraph{Candidate recall and query calibration.}
The candidate set $\mathcal C_u$, calibrated scores $z_{u,x}$, and
active set $\mathcal A_{u,x}$ are defined in
Equation~\ref{eq:lineup_routing}. Softmax over active scores gives
\begin{equation}
w_{u,x,k}
=
\begin{cases}
\displaystyle
\frac{\exp(z_{u,x,k})}
     {\sum_{j\in\mathcal A_{u,x}}\exp(z_{u,x,j})},
& k\in\mathcal A_{u,x},\\[5pt]
0,& k\notin\mathcal A_{u,x}.
\end{cases}
\label{eq:lineup_weights}
\end{equation}
These weights yield the update in Equation~\ref{eq:lineup_composition},
replacing $W_i h$ by $(W_i+D_{u,x,i})h$. One factor mixture applies across
all adapted sites and every prompt or generated token of a request.
When $M_a=M_c$, query calibration reweights the recalled factors; when
$M_a<M_c$, it can also change the active subset. Prediction gradients
update active factors and routing scores through the selected-set softmax;
TopK indices are treated as discrete choices.

\subsection{Tiny User Codes and Support Adaptation}
\label{app:lineup_code_adaptation}

\paragraph{Shared correction maps.}
Equation~\ref{eq:lineup_correction} uses a shared map
$V_\ell\in\mathbb R^{d\times r_{\mathrm c}}$ at each selected block
output $\ell\in\mathcal L_{\mathrm{corr}}$ and one user code shared
across these layers. Any fixed multiplicative scaling is absorbed into $V_\ell$, so $V_\ell$ denotes the effective shared correction map. For a general forward pass, replace $\hat c_u$ by
$c_u^{(0)}$ during shared S2 training, $c$ during support fitting, and
$\hat c_u$ during target-query prediction. The binary mask $a_{\ell,t}$
selects correction positions. In multi-token passes, it selects configured
boundary windows of the non-padding sequence: prompt plus output for
teacher forcing, and prompt only for generation prefill. Cached decoding
applies the correction to each newly processed token. Layer choices and
window sizes are specified with task configurations.

\paragraph{History-informed initialization.}
The shared linear map $G\in\mathbb R^{r_{\mathrm c}\times d_f}$ produces
$c_u^{(0)}=G\tilde e_u$ from the same history encoding used by S1.
Source-user training learns $G$ and
$\{V_\ell\}_{\ell\in\mathcal L_{\mathrm{corr}}}$ through prediction
loss at these initialized codes. Target-user support fitting refines this
starting point while all shared maps remain fixed.

\paragraph{Support construction.}
A task-specific construction maps selected records from the available
history to supervised input--output pairs $\mathcal S_u$. History encoding
uses all available records, independently of how many are selected for
fitting. Under limited-history evaluation, both conditioning and support
construction use only the available history.

\paragraph{Code refinement.}
Write $\widehat{\mathcal L}_{\mathcal S_u}(c)
=\mathcal L_{\mathrm{tok}}(\mathcal S_u;f_2(c))$.
The fitting objective is
\begin{equation}
\mathcal J_u(c)
=
\widehat{\mathcal L}_{\mathcal S_u}(c)
+\lambda_{\mathrm{reg}}\|c-c_u^{(0)}\|_2^2.
\label{eq:lineup_support_fitting}
\end{equation}
Cross-entropy is averaged over supervised output tokens pooled across
support examples, rather than averaging example losses. The regularizer
penalizes squared deviation from the fixed history-based initialization.

S1 remains active: each support input receives its own query-calibrated
factor mixture using the unchanged history encoding. Fitting uses the
same routing and correction configuration as target-query prediction.
A fresh Adam state is initialized for each user, and $T$ full-support
updates are taken from $c_u^{(0)}$. Only the code is updated; the backbone
and all shared parameters remain fixed. The final iterate $\hat c_u$ is
reused across subsequent queries, while S1 recomputes their mixtures from
the same history and their respective inputs. We use $r_{\mathrm c}=8$.
S1-only evaluation omits correction and code fitting; full LINEUP includes
history-based initialization followed by support refinement.

%% file: Appendix/D_Theoretical_Analysis_and_Proofs.tex
\section{Theoretical Analysis and Proofs}
\label{app:s2_theory}

We analyze the finite-step S2 estimator conditional on all source-learned
parameters; a quadratic specialization explains the learned geometry.

\subsection{History-Initialized Fitting and Query Risk}
\label{app:s2_risk}

\paragraph{Setup.}
Fix a user with $N$ history records $\mathcal H$ and $n$ support examples.
Stack the shared correction maps into $V\in\mathbb R^{D\times r}$,
where $D=|\mathcal L_{\mathrm{corr}}|d$ and $r=r_{\mathrm c}$.
One code controls all selected layers, fixed scaling is included in $V$,
and S1 remains history--query conditioned.
For the token-pooled support NLL $\widehat L_{\mathcal H}$ in
Equation~\ref{eq:lineup_support_fitting}, write
\begin{equation}
c_0=G\tilde e(\mathcal H),\qquad
F_{\mathcal H}(c)=\widehat L_{\mathcal H}(c)
+\frac{\rho}{2}\|c-c_0\|_2^2,\qquad
\rho=2\lambda_{\mathrm{reg}}\geq0.
\label{eq:s2_theory_objective}
\end{equation}
Let $c_T$ be the final code and $d_T=c_T-c_0$.
For expected query NLL write
$\mathcal R_{\mathcal H}(c)=\mathcal R^{\mathrm{rep}}_{\mathcal H}(Vc)$
and $\mathcal R_{\mathcal H}^{\star}=\inf_v\mathcal R^{\mathrm{rep}}_{\mathcal H}(v)$:
$v$ has the same correction scope, not arbitrary backbone updates.
The query distribution and loss weighting are fixed and may differ from
support. With finite risks, define
\begin{equation}
\mathcal A_{\mathcal H}
=\inf_c\mathcal R_{\mathcal H}(c)-\mathcal R_{\mathcal H}^{\star},
\qquad
\mathcal I_{\mathcal H}
=\mathcal R_{\mathcal H}(c_0)-\inf_c\mathcal R_{\mathcal H}(c).
\label{eq:s2_theory_errors}
\end{equation}
These are shared-space approximation and initialization errors.
Let $\Delta_T=F_{\mathcal H}(c_0)-F_{\mathcal H}(c_T)$ and
\begin{equation}
\epsilon_{\mathcal H}(R)
=\sup_{\|d\|_2\leq R}
\|\nabla\mathcal R_{\mathcal H}(c_0+d)
-\nabla\widehat L_{\mathcal H}(c_0+d)\|_2.
\label{eq:s2_gradient_discrepancy}
\end{equation}

\begin{proposition}[Finite-step comparison]
\label{prop:s2_pathwise}
Suppose the losses are continuously differentiable on the segment from
$c_0$ to $c_T$, $\|d_T\|_2\leq R$, and $\rho>0$. Then
\begin{equation}
\mathcal R_{\mathcal H}(c_T)-\mathcal R_{\mathcal H}^{\star}
\leq\mathcal A_{\mathcal H}+\mathcal I_{\mathcal H}-\Delta_T
+\frac{\epsilon_{\mathcal H}(R)^2}{2\rho}.
\label{eq:s2_pathwise_risk}
\end{equation}
\end{proposition}
\begin{proof}
Integrating $E_{\mathcal H}=\mathcal R_{\mathcal H}-\widehat L_{\mathcal H}$
along the segment yields
$E_{\mathcal H}(c_T)-E_{\mathcal H}(c_0)
\leq\epsilon_{\mathcal H}(R)\|d_T\|_2$.
Also,
$\widehat L_{\mathcal H}(c_T)-\widehat L_{\mathcal H}(c_0)
=-\Delta_T-\rho\|d_T\|_2^2/2$.
Combine these relations with
$\mathcal R_{\mathcal H}(c_0)-\mathcal R_{\mathcal H}^{\star}
=\mathcal A_{\mathcal H}+\mathcal I_{\mathcal H}$, then use
$-\rho x^2/2+\epsilon x\leq\epsilon^2/(2\rho)$.
\end{proof}

\subsection{Finite History and Shared Conditioning}
\label{app:s2_history}

Define
$Z_{\mathcal H}(d)=\nabla\mathcal R_{\mathcal H}(c_0(\mathcal H)+d)
-\nabla\widehat L_{\mathcal H}(c_0(\mathcal H)+d)$.
Record replacement acts on conditioning, initialization, and support
jointly; initialization is not assumed independent of support.

\begin{proposition}[Uniform gradient discrepancy]
\label{prop:s2_history_bound}
Condition on the user and shared parameters. Assume independent records
and square-integrable, $L_Z$-Lipschitz $Z_{\mathcal H}$ on the radius-$R$
ball. Uniformly over this ball and permitted record replacements, assume
\begin{equation}
\|\mathbb E Z_{\mathcal H}(d)\|_2\leq b,\qquad
\|Z_{\mathcal H}(d)-Z_{\mathcal H^{(i)}}(d)\|_2\leq a_i.
\label{eq:s2_history_conditions}
\end{equation}
Here $\mathcal H^{(i)}$ replaces record $i$, and all bounds are
deterministic. Set $S_a=\sum_i a_i^2$. For fixed $s>0$ and
$\delta\in(0,1)$, with probability at least $1-\delta$,
\begin{equation}
\epsilon_{\mathcal H}(R)\leq
\overline\epsilon(R,s,\delta)
:=b+\sqrt{S_a/2}+L_Zs
+\sqrt{\frac{S_a}{2}
\left[r\log\left(1+\frac{2R}{s}\right)+\log\frac1\delta\right]}.
\label{eq:s2_history_bound}
\end{equation}
\end{proposition}
\begin{proof}
For fixed $d$, vector Efron--Stein bounds the expected centered norm by
$\sqrt{S_a/2}$; McDiarmid gives tail probability $\exp(-2t^2/S_a)$
at $t$ above this mean \citep{DBLP:books/daglib/0035704}.
Apply a union bound over an $s$-net of size at most
$(1+2R/s)^r$, add $b$, and extend to the ball using $L_Z$-Lipschitzness.
The case $S_a=0$ is deterministic.
\end{proof}

\paragraph{Record-level scaling.}
For i.i.d. records, a data-independent subset of $n\leq N$ support
records, bounded output-length ratios, per-example code gradients
bounded by $B$, and gradient sensitivity $O(\kappa/N)$ to replacing a
record in conditioning and initialization, a replace-one argument yields
$S_a=O(\kappa^2/N+B^2/n)$ and
$b=b_{\mathrm{shift}}+O(\kappa/N+B/\sqrt n)$, where
$b_{\mathrm{shift}}$ denotes population gradient mismatch.

\paragraph{Routing stability.}
With bounded record features and centered encoding norms bounded away
from zero, history averaging changes scores by at most
$\Delta_s=O(1/N)$ under record replacement, and TopK sets are preserved
when each nontrivial selection gap exceeds $2\Delta_s$.
Otherwise routing jumps must enter the influence constants rather than
assuming the same $O(1/N)$ sensitivity across changes of selected indices.

\subsection{Finite-Step Adam and the Combined Bound}
\label{app:s2_adam}

\paragraph{Updates.}
For full-support Adam with fresh moments \citep{DBLP:journals/corr/KingmaB14},
$g_t=\nabla F_{\mathcal H}(c_{t-1})$ and
\begin{equation}
\begin{aligned}
m_t&=\beta_1m_{t-1}+(1-\beta_1)g_t,
&s_t&=\beta_2s_{t-1}+(1-\beta_2)g_t^{\odot2},\\
\hat m_t&=m_t/(1-\beta_1^t),
&\hat s_t&=s_t/(1-\beta_2^t),\\
M_t&=\operatorname{diag}((\sqrt{\hat s_t}+\varepsilon_{\mathrm A})^{-1}),
&c_t&=c_{t-1}-\alpha_tM_t\hat m_t.
\end{aligned}
\label{eq:s2_adam_update}
\end{equation}
Here $m_0=s_0=0$, $0<\beta_1,\beta_2<1$, and
$\varepsilon_{\mathrm A}>0$; elementwise operations are understood.
Learning rates $\alpha_t>0$ are fixed before target-user fitting.

\begin{lemma}[Deterministic displacement bound]
\label{lem:s2_adam_radius}
For every $t\leq T$,
\begin{equation}
\begin{aligned}
\|c_t-c_0\|_2&\leq R_T:=\sqrt r\sum_{j=1}^{T}\alpha_j\sqrt{\chi_j},\\
\chi_t&:=\frac{(1-\beta_1)^2(1-\beta_2^t)}
{(1-\beta_1^t)^2(1-\beta_2)}
\sum_{j=0}^{t-1}\left(\frac{\beta_1^2}{\beta_2}\right)^j.
\end{aligned}
\label{eq:s2_adam_radius}
\end{equation}
Each step also satisfies $\|c_t-c_{t-1}\|_2\leq\alpha_t\sqrt{r\chi_t}$.
\end{lemma}
\begin{proof}
Apply weighted Cauchy--Schwarz to the bias-corrected moment sums in
Equation~\ref{eq:s2_adam_update}: for each coordinate $k$,
$\hat m_{t,k}^2\leq\chi_t\hat s_{t,k}$.
The normalized step in that coordinate is therefore at most
$\alpha_t\sqrt{\chi_t}$, including the zero-moment case.
Sum squared coordinate bounds and then step lengths.
\end{proof}

\paragraph{Numerical scale and observables.}
For illustration, $(\beta_1,\beta_2)=(0.9,0.999)$, $r=8$, and $T=10$
give $1\leq\sqrt{\chi_t}<1.044$ for $1\leq t\leq10$, and
$R_{10}\simeq28.752\alpha$ for constant learning rate. The reported code learning rates $0.001$--$0.008$
(Table~\ref{tab:lineup_hparams}) yield bounds $0.029$--$0.230$
under these moment settings. Actual $\|d_T\|_2$ and $\Delta_T$ are
observable, but query improvement still requires control of
$\epsilon_{\mathcal H}(R_T)$; the numerical bounds are not measurements.

\begin{proposition}[Progress along Adam]
\label{prop:s2_adam_progress}
If $F_{\mathcal H}$ has an $L_F$-Lipschitz gradient on the trajectory
ball, then
\begin{equation}
\Delta_T\geq\Gamma_T
:=\sum_{t=1}^{T}
\left[\alpha_tg_t^\top M_t\hat m_t
-\frac{L_F\alpha_t^2}{2}\|M_t\hat m_t\|_2^2\right].
\label{eq:s2_adam_progress}
\end{equation}
\end{proposition}
\begin{proof}
Apply the smoothness inequality to each update in
Equation~\ref{eq:s2_adam_update} and sum.
\end{proof}
\noindent\emph{Momentum remark.}
If $\alpha_tL_F\|M_t\|_{\mathrm{op}}\leq1$, the per-step decrease is
at least $\alpha_t(\|g_t\|_{M_t}^2-\|e_t\|_{M_t}^2)/2$, where
$e_t=\hat m_t-g_t$ and $\|x\|_M^2=x^\top Mx$.
As $e_t$ is a convex combination of previous gradient differences,
smoothness and Lemma~\ref{lem:s2_adam_radius} bound its norm through
preceding step lengths.

\begin{corollary}[Finite-history risk of S2]
\label{cor:s2_finite_history}
Under Propositions~\ref{prop:s2_history_bound} and
\ref{prop:s2_adam_progress} with $R=R_T$, set
$\bar\epsilon_T=\overline\epsilon(R_T,s,\delta)$.
For $\rho>0$, with probability at least $1-\delta$,
\begin{equation}
\mathcal R_{\mathcal H}(c_T)-\mathcal R_{\mathcal H}^{\star}
\leq\mathcal A_{\mathcal H}+\mathcal I_{\mathcal H}-\Gamma_T
+\frac{\bar\epsilon_T^2}{2\rho}.
\label{eq:s2_finite_history_risk}
\end{equation}
The sharper bound with $\Delta_T$ replacing $\Gamma_T$ also holds;
in particular, fitting improves on initialization if
$\Delta_T+\rho\|d_T\|_2^2/2>\bar\epsilon_T\|d_T\|_2$.
For $\rho=0$, the last term in
Equation~\ref{eq:s2_finite_history_risk} becomes $\bar\epsilon_TR_T$.
\end{corollary}
\begin{proof}
Lemma~\ref{lem:s2_adam_radius} supplies a deterministic ball, so the
uniform discrepancy bound applies to its history-dependent output.
Use the proof of Proposition~\ref{prop:s2_pathwise} and
$\Delta_T\geq\Gamma_T$.
\end{proof}

\subsection{Correction Geometry and Transfer}
\label{app:s2_quadratic}
\label{app:s2_transfer}

\paragraph{Quadratic specialization.}
Let $q(v)=q(v^\star)+\|v-v^\star\|_H^2/2$ with $H\succ0$, and
fix $Q=V^\top HV$. The optimal displacement from initialization is
$d^\star=Q^\dagger V^\top H(v^\star-Vc_0)$, where $\dagger$ denotes
the Moore--Penrose pseudoinverse, and define
$\mathcal A_q(V)=\inf_cq(Vc)-q(v^\star)$.
Consider a support model with the same curvature,
$\widehat q(d)=\mathrm{const}-(Qd^\star+\xi)^\top d+d^\top Qd/2$,
and its exact regularized solution
$\bar d=\arg\min_d\{\widehat q(d)+\rho\|d\|_2^2/2\}$, $\rho>0$.
Write $q_j>0$ for the positive eigenvalues of $Q$ and $d_j^\star,\xi_j$
for the corresponding coordinates.

\begin{corollary}[Geometry-controlled refinement]
\label{cor:s2_effective_dimension}
With finite second moments, suppose
$\mathbb E\xi=0$,
$\mathbb E[\xi\xi^\top]=\sigma^2Q/n$, and
$\mathbb E[d_j^\star\xi_j]=0$ for every $q_j>0$.\footnote{Without decorrelation, add
$-\rho\sum_{j:q_j>0}q_j\mathbb E[d_j^\star\xi_j]/(q_j+\rho)^2$
to the risk; the directional criterion then changes.}
Then
\begin{equation}
\begin{aligned}
\mathbb E[q(V(c_0+\bar d))-q(v^\star)]
={}&\mathbb E\mathcal A_q(V)
+\frac12\sum_{j:q_j>0}
q_j\left(\frac{\rho}{q_j+\rho}\right)^2\mathbb E[(d_j^\star)^2]
+\frac{\sigma^2}{2n}d_{\mathrm{eff}}(\rho),\\
d_{\mathrm{eff}}(\rho)
:={}&\sum_{j:q_j>0}\left(\frac{q_j}{q_j+\rho}\right)^2
\leq\operatorname{rank}(V)\leq r.
\end{aligned}
\label{eq:s2_effective_dimension}
\end{equation}
Refinement reduces expected query error in direction $j$, relative to
retaining initialization, exactly when
\begin{equation}
(q_j+2\rho)\,\mathbb E[(d_j^\star)^2]>\frac{\sigma^2}{n}.
\label{eq:s2_directional_refinement}
\end{equation}
\end{corollary}
\begin{proof}
The regularized normal equation gives
$\bar d-d^\star=(Q+\rho I)^{-1}(\xi-\rho d^\star)$.
Weighted projection gives
$q(V(c_0+d))-q(v^\star)=\mathcal A_q(V)+\|d-d^\star\|_Q^2/2$.
Substitute the code error, diagonalize $Q$, and apply the stated moments.
Each effective-dimension summand is
at most one, and $\operatorname{rank}(Q)=\operatorname{rank}(V)$.
Subtracting refined from initialized directional risk gives
$\frac{q_j^2}{2(q_j+\rho)^2}
[(q_j+2\rho)\mathbb E[(d_j^\star)^2]-\sigma^2/n]$,
proving the criterion.
\end{proof}

\paragraph{Transfer through the learned maps.}
At $v_0=Vc_0$, let $g_q$ and $g_s$ be negative representation gradients
of query and support NLL: fresh Adam moments and the zero regularizer
gradient give $c_1-c_0=\alpha_1M_1V^\top g_s$.
For $L_R$-smooth query risk, its decrease is at least
$\alpha_1g_q^\top VM_1V^\top g_s
-L_R\alpha_1^2\|M_1V^\top g_s\|_2^2/2$.
Thus the positive semidefinite, rank-at-most-$r$ operator $VM_1V^\top$
governs transfer, connecting Observation~\hyperref[obs:correcting]{3}
to learned correction geometry rather than an orthogonal projector.

\paragraph{Shared learning.}
On source episodes,
$\mathcal R_{\mathcal H}(c_0)=\mathcal R_{\mathcal H}^{\star}
+\mathcal A_{\mathcal H}+\mathcal I_{\mathcal H}$,
so the shared S2 population objective targets approximation and
initialization error jointly, without a support-fitting inner loop.

\paragraph{Scope.}
The finite-step NLL bounds use the stated smoothness and influence
conditions, not Adam convergence.
Independent records are needed only for concentration, not the pathwise
comparison or displacement bound.
The quadratic corollary additionally uses common curvature and stated
moments; it does not predict an optimal code size or gains in every metric.

%% file: Appendix/E_Supplementary_Materials_for_Experimental_Setup.tex
\section{Supplementary Materials for Experimental Setup}
\label{app:experimental_setup}

\subsection{Datasets}
\label{app:datasets}

We use six LaMP tasks \citep{salemi-etal-2024-lamp} under OPPU's protocol of data processing, target-user partition, and prompt construction \citep{tan-etal-2024-democratizing}. Table~\ref{tab:dataset_stats} summarizes their statistics. Following prior OPPU-based evaluations \citep{ICLR2026_7a1ac10e, li-etal-2026-mta}, we exclude LaMP-1 because outputs under the OPPU prompt and Llama-2-7B frequently violate the required citation format, hindering reliable evaluation; we omit LaMP-6 because of data privacy concerns. The remaining six datasets still ensure task diversity.

We further partition the source users into fixed 90/10 training/validation sets. These sets and the target-user cohort are mutually disjoint, and the split is unchanged across seeds. Target-user histories are used for personalization, while their held-out queries are reserved for final evaluation.

\begin{table}[t]
    \setlength{\tabcolsep}{3.8pt}
    \renewcommand{\arraystretch}{1.12}

    \caption{Statistics of the six LaMP tasks. Avg.\ history is the mean number of history records per target user.}
    \label{tab:dataset_stats}

    \begin{center}
    \resizebox{\linewidth}{!}{%
        \begin{tabular}{@{}llcccccc@{}}
            \toprule
            \textbf{Dataset}
            & \textbf{Task}
            & \textbf{Type}
            & \shortstack{\textbf{Users}\\\textbf{(source / target)}}
            & \shortstack{\textbf{Source}\\\textbf{queries}}
            & \shortstack{\textbf{Avg.}\\\textbf{history}}
            & \shortstack{\textbf{Target}\\\textbf{queries}}
            & \shortstack{\textbf{\#}\\\textbf{classes}} \\
            \midrule

            LaMP-2M & Movie tagging & Classification
            & 829 / 100 & 3,181 & 55.6 & 3,302 & 15 \\

            LaMP-2N & News categorization & Classification
            & 274 / 49 & 3,662 & 219.9 & 6,033 & 15 \\

            LaMP-3 & Product rating & Prediction
            & 19,899 / 101 & 22,388 & 959.8 & 112 & 5 \\

            LaMP-4 & News headline generation & Generation
            & 1,543 / 100 & 7,275 & 270.1 & 6,725 & -- \\

            LaMP-5 & Scholarly title generation & Generation
            & 14,581 / 101 & 16,075 & 442.9 & 107 & -- \\

            LaMP-7 & Tweet paraphrasing & Generation
            & 13,337 / 100 & 14,826 & 121.2 & 109 & -- \\

            \bottomrule
        \end{tabular}%
    }
    \end{center}
\end{table}

\subsection{Evaluation Metrics}
\label{app:metrics}

\begin{table}[t]
\setlength{\tabcolsep}{4.2pt}
\renewcommand{\arraystretch}{1.12}

\caption{
Evaluation metrics used across the six LaMP tasks. \textbf{Acc} = Accuracy, \textbf{F1} = Macro-F1, \textbf{R-1} = ROUGE-1, and \textbf{R-L} = ROUGE-L. Here, $P_c$ and $R_c$ denote precision and recall for class $c$; $P_{1,i}$ and $R_{1,i}$ denote unigram precision and recall for the
$i$-th example; and $P_{L,i}$ and $R_{L,i}$ denote the corresponding precision and recall based on the longest common subsequence.
}
\label{tab:metrics}
\begin{center}
\begin{tabular}{@{}llll@{}}
\toprule
\textbf{Metric}
& \textbf{Datasets}
& \textbf{Task type}
& \textbf{Formulation} \\
\midrule

\textbf{Acc} $\uparrow$
& LaMP-2M, LaMP-2N
& Classification
& $\displaystyle
\frac{1}{n}\sum_{i=1}^{n}\mathbb{I}[\hat y_i=y_i]
$ \\[4pt]

\textbf{F1} $\uparrow$
& LaMP-2M, LaMP-2N
& Classification
& $\displaystyle
\frac{1}{C}\sum_{c=1}^{C}
\frac{2P_cR_c}{P_c+R_c}
$ \\[6pt]

\textbf{MAE} $\downarrow$
& LaMP-3
& Rating prediction
& $\displaystyle
\frac{1}{n}\sum_{i=1}^{n}|y_i-\hat y_i|
$ \\[6pt]

\textbf{RMSE} $\downarrow$
& LaMP-3
& Rating prediction
& $\displaystyle
\sqrt{\frac{1}{n}\sum_{i=1}^{n}(y_i-\hat y_i)^2}
$ \\[7pt]

\textbf{R-1} $\uparrow$
& LaMP-4, LaMP-5, LaMP-7
& Generation
& $\displaystyle
\frac{1}{n}\sum_{i=1}^{n}
\frac{2P_{1,i}R_{1,i}}{P_{1,i}+R_{1,i}}
$ \\[7pt]

\textbf{R-L} $\uparrow$
& LaMP-4, LaMP-5, LaMP-7
& Generation
& $\displaystyle
\frac{1}{n}\sum_{i=1}^{n}
\frac{2P_{L,i}R_{L,i}}{P_{L,i}+R_{L,i}}
$ \\

\bottomrule
\end{tabular}
\end{center}
\end{table}

All methods are scored by the same evaluator using the metrics in Table~\ref{tab:metrics}.

For generation, we report ROUGE-1 and ROUGE-L F-measures averaged over examples with one reference, using \texttt{rouge-score} with Porter stemming and \texttt{rougeL}. For LaMP-3, numeric predictions are parsed as floats without rounding or clipping. Unparsable outputs are assigned the farther endpoint in $\{1,5\}$ from the reference rating, following the official LaMP evaluator.

\subsection{Baselines}
\label{app:baseline_impl}

We use baselines' official implementations. Their personalization algorithms are retained under the common Llama-2-7B backbone and OPPU target-user protocol. 

Source-data construction and shared training remain method-specific to ensure fairness. OPPU and PerFit perform task adaptation on the full source pool. P2P trains its hypernetwork on all source users, conditioning adapter generation on user profiles and query-retrieved history. Per-Pcs and MTA retain their released selection procedures, choosing sharers or anchors from up to 1,000 candidates drawn from the full source partition. RAG retrieves the top user-history record using BM25.

Per-Pcs and MTA provide no complete released configuration for LaMP-2N. MTA also excludes LaMP-7 because its adaptation-data construction requires query--answer pairs. These entries are marked as NA.

\subsection{Implementation}
\label{app:implementation}

We learn LINEUP's shared components on source-training users and select hyperparameters on source-validation users.
Selection uses the arithmetic mean of the two task metrics, maximized for classification and generation and minimized for LaMP-3. The selected configurations are fixed before final target-user evaluation.

We report mean $\pm$ standard deviation over three runs. LINEUP and baselines are independently trained with three seeds.

\subsection{Hyperparameters}
\label{app:hyperparameters}

Across tasks, LINEUP uses $K=32$ rank-4 LoRA factors over the query and
value projections of all transformer layers, with 256-dimensional history
and query representations.
Target-user fitting updates only an eight-dimensional code for 10 Adam steps. Table~\ref{tab:lineup_hparams} lists
detailed settings.

\begin{table}[t]
    \setlength{\tabcolsep}{4.0pt}
    \renewcommand{\arraystretch}{1.12}

    \caption{LINEUP hyperparameters.
    ``F7+L7'' selects the first and last seven token positions; ``L7'' selects
    the last seven.}
    \label{tab:lineup_hparams}

    \begin{center}
    \resizebox{\linewidth}{!}{%
        \begin{tabular}{@{}lcccccc@{}}
            \toprule
            \textbf{Hyperparameter}
            & \textbf{LaMP-2M}
            & \textbf{LaMP-2N}
            & \textbf{LaMP-3}
            & \textbf{LaMP-4}
            & \textbf{LaMP-5}
            & \textbf{LaMP-7} \\
            \midrule

            Candidate / active factors
            & 4 / 4 & 4 / 4 & 4 / 4 & 16 / 4 & 16 / 4 & 16 / 4 \\

            Query strength $\lambda$
            & 2 & 2 & 2 & 8 & 8 & 8 \\

            S1 learning rate
            & $2{\times}10^{-4}$
            & $1{\times}10^{-4}$
            & $2{\times}10^{-4}$
            & $5{\times}10^{-5}$
            & $2{\times}10^{-4}$
            & $2{\times}10^{-4}$ \\

            S2 layers
            & 14, 15 & 14--16 & 14, 15 & 14--16 & 15 & 15, 16 \\

            S2 positions
            & L7 & F7+L7 & L7 & F7+L7 & L7 & F7+L7 \\

            S2 shared learning rate
            & $2.25{\times}10^{-3}$
            & $1{\times}10^{-3}$
            & $2.25{\times}10^{-3}$
            & $1{\times}10^{-3}$
            & $1{\times}10^{-3}$
            & $1{\times}10^{-3}$ \\

            User-code learning rate
            & $2{\times}10^{-3}$
            & $2{\times}10^{-3}$
            & $4{\times}10^{-3}$
            & $8{\times}10^{-3}$
            & $1{\times}10^{-3}$
            & $1{\times}10^{-3}$ \\

            \bottomrule
        \end{tabular}%
    }
    \end{center}
\end{table}

%% file: Appendix/F_Supplementary_Experimental_Results.tex
\section{Supplementary Experimental Results}
\label{app:supplementary_results}

\subsection{Long-Form Generation on LongLaMP}
\label{app:longlamp}

We extend our evaluation to long-form personalized generation on three LongLaMP tasks \citep{kumar2024longlampbenchmarkpersonalizedlongform}:
LongLaMP-2 (Abstract Generation), LongLaMP-3 (Product Review Generation), and LongLaMP-4 (Topic Writing).

\subsubsection{Experimental Setup}
\label{app:longlamp_datasets}

We first consolidate the official LongLaMP Temporal release at the user level, yielding one available history and a set of held-out queries for each user. We select users with the longest histories as the target cohort. The remaining users form the source pool and are eligible for LINEUP's training/validation split. User-disjoint splitting and model selection follow Appendix~\ref{app:datasets} and Appendix~\ref{app:implementation}. LongLaMP-1 (Email Completion) is omitted because its underlying Avocado email corpus requires licensed access.

For history-derived supervision, Abstract Generation uses title--abstract pairs, Product Review Generation uses ratings, descriptions, and summaries to predict review text, and Topic Writing uses summary--content pairs. Table~\ref{tab:longlamp_dataset_stats} summarizes the resulting datasets.

For this supplementary evaluation, we compare Non, Basic, RAG, and OPPU under the same user split and evaluation protocol. These cover the pretrained model, task adaptation, retrieval-based personalization, and full private user adapters. Extending PerFit, Per-Pcs, and MTA would require substantial LongLaMP-specific adaptation. P2P reports LongLaMP results under a different protocol, so aligning it with our user split would require rebuilding user representations and retraining. We therefore retain a focused comparison, with OPPU providing the direct private-adapter reference central to our research question.

\begin{table}[t]
    \setlength{\tabcolsep}{3.8pt}
    \renewcommand{\arraystretch}{1.12}

    \caption{Statistics of the three LongLaMP tasks. Avg.\ history is the mean number of history records per target user.}
    \label{tab:longlamp_dataset_stats}

    \begin{center}
    \resizebox{\linewidth}{!}{%
        \begin{tabular}{@{}llcccc@{}}
            \toprule
            \textbf{Dataset}
            & \textbf{Task}
            & \shortstack{\textbf{Users}\\\textbf{(source / target)}}
            & \shortstack{\textbf{Source}\\\textbf{queries}}
            & \shortstack{\textbf{Avg.}\\\textbf{history}}
            & \shortstack{\textbf{Target}\\\textbf{queries}} \\
            \midrule

            LongLaMP-2
            & Abstract generation
            & 22,620 / 101
            & 31,663
            & 1,291.3
            & 145 \\

            LongLaMP-3
            & Product review generation
            & 15,997 / 100
            & 19,561
            & 560.5
            & 128 \\

            LongLaMP-4
            & Topic writing
            & 16,136 / 100
            & 20,975
            & 260.6
            & 132 \\

            \bottomrule
        \end{tabular}%
    }
    \end{center}
\end{table}

\subsubsection{Results on LongLaMP}
\label{app:longlamp_results}

Table~\ref{tab:main_results_longlamp} shows that \textbf{LINEUP achieves the best mean performance on all six metrics across the three tasks}, with OPPU providing the strongest baseline throughout. The largest gains occur on Product Review Generation, where LINEUP improves over OPPU by
$0.056$ in ROUGE-1 and $0.048$ in ROUGE-L.

The advantage over OPPU extends our capacity-allocation findings to long-form generation. Jointly learned factors provide reusable personalization capacity, history--query composition tailors it to each request, and the compact user code supplies residual individual correction. Thus, long-form personalization can also benefit from rich shared capacity without requiring a complete private adapter for every user.

\begin{table}[t]
\setlength{\tabcolsep}{3.6pt}
\renewcommand{\arraystretch}{1.12}

\caption{Performance on LongLaMP-2, LongLaMP-3, and LongLaMP-4.
\textbf{R-1} = ROUGE-1 and \textbf{R-L} = ROUGE-L. \textbf{Bold} and \underline{underline} denote the best and second-best means. $\Delta$ reports absolute improvements over the strongest baseline. LINEUP leads on all metrics for long-form generation.}
\label{tab:main_results_longlamp}
\begin{center}
\begin{tabular}{@{}lcccccc@{}}
\toprule
& \multicolumn{2}{c}{\textbf{LongLaMP-2}}
& \multicolumn{2}{c}{\textbf{LongLaMP-3}}
& \multicolumn{2}{c}{\textbf{LongLaMP-4}} \\
\cmidrule(lr){2-3}
\cmidrule(lr){4-5}
\cmidrule(lr){6-7}

\textbf{Method}
& \textbf{R-1} $\uparrow$
& \textbf{R-L} $\uparrow$
& \textbf{R-1} $\uparrow$
& \textbf{R-L} $\uparrow$
& \textbf{R-1} $\uparrow$
& \textbf{R-L} $\uparrow$ \\
\midrule

Non
& $.195 \pm .003$
& $.121 \pm .002$
& $.143 \pm .010$
& $.082 \pm .004$
& $.174 \pm .005$
& $.090 \pm .002$ \\

Basic
& $.367 \pm .010$
& $.194 \pm .006$
& $.296 \pm .001$
& $.112 \pm .001$
& $.233 \pm .009$
& $.103 \pm .002$ \\

RAG
& $.381 \pm .005$
& $.211 \pm .002$
& $.322 \pm .005$
& $.135 \pm .003$
& $.242 \pm .006$
& $.114 \pm .004$ \\

OPPU
& $\underline{.412} \pm .003$
& $\underline{.244} \pm .004$
& $\underline{.346} \pm .002$
& $\underline{.157} \pm .001$
& $\underline{.260} \pm .003$
& $\underline{.123} \pm .001$ \\

\rowcolor{cyan!10}
\textbf{LINEUP}
& $\mathbf{.435} \pm .006$
& $\mathbf{.263} \pm .005$
& $\mathbf{.402} \pm .004$
& $\mathbf{.205} \pm .001$
& $\mathbf{.289} \pm .003$
& $\mathbf{.136} \pm .002$ \\

\rowcolor{cyan!10}
$\Delta$
& $\textbf{\textcolor{tablegreen}{+0.023}}$
& $\textbf{\textcolor{tablegreen}{+0.019}}$
& $\textbf{\textcolor{tablegreen}{+0.056}}$
& $\textbf{\textcolor{tablegreen}{+0.048}}$
& $\textbf{\textcolor{tablegreen}{+0.029}}$
& $\textbf{\textcolor{tablegreen}{+0.013}}$ \\

\bottomrule
\end{tabular}
\end{center}
\end{table}

\subsection{Additional Cold-Start Results}
\label{app:coldstart_results}

\begin{figure}[t]
    \centering

    \begin{minipage}[t]{0.32\linewidth}
        \centering
        \includegraphics[width=\linewidth]{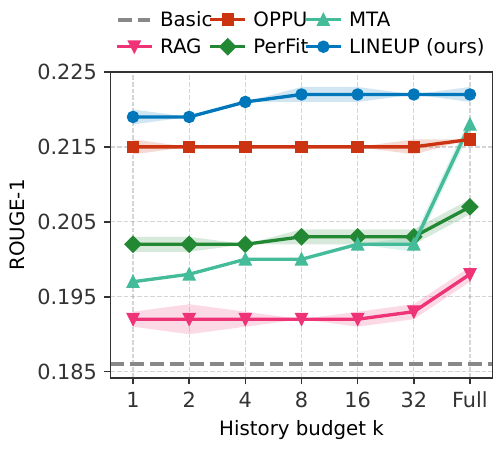}\\[-1mm]
        (a) LaMP-4 ROUGE-1
    \end{minipage}\hfill%
    \begin{minipage}[t]{0.32\linewidth}
        \centering
        \includegraphics[width=\linewidth]{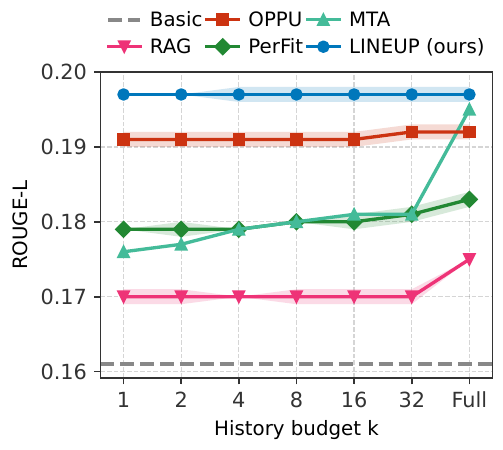}\\[-1mm]
        (b) LaMP-4 ROUGE-L
    \end{minipage}\hfill%
    \begin{minipage}[t]{0.32\linewidth}
        \centering
        \includegraphics[width=\linewidth]{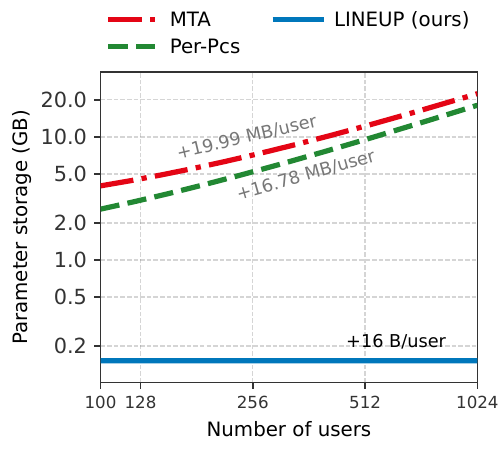}\\[-1mm]
        (c) LaMP-3 storage
    \end{minipage}

    \caption{Supplementary cold-start and adaptation-storage results.
    (a) One interaction nearly recovers LINEUP's full-history ROUGE-1.
    (b) LINEUP maintains the highest ROUGE-L across history budgets.
    (c) LINEUP's storage grows far more slowly than Per-Pcs and MTA.
    }
    \label{fig:coldstart_efficiency_supp}
\end{figure}

\paragraph{Protocol.}
We evaluate LaMP-2N and LaMP-4 using nested prefixes of each user's earliest
$k\in\{1,2,4,8,16,32,\mathrm{Full}\}$ historical interactions. Target
users and evaluation queries remain fixed. Shared parameters are unchanged,
and every history-dependent operation uses only the current prefix.
LINEUP recomputes history-conditioned recall and code initialization at
each budget, then fits a new code using available support examples.
The full available prefix is used for history encoding.
Results are mean $\pm$ sample standard deviation over three seeds.
Basic is history-independent, so its full-history result is repeated across
budgets. Tables~\ref{tab:coldstart_lamp2n} and~\ref{tab:coldstart_lamp4}
report the complete results.

\paragraph{Generation results.}
Figure~\ref{fig:coldstart_efficiency_supp}(a--b) and
Table~\ref{tab:coldstart_lamp4} provide complementary results on LaMP-4.
LINEUP achieves the best mean ROUGE-1 and ROUGE-L at every history budget.
With one history record, it reaches $0.219/0.197$, compared with OPPU's
$0.215/0.191$ at the same budget. Several baselines also change little with
additional history; the key finding here is LINEUP's consistently stronger
generation performance under restricted history.

\paragraph{Baseline coverage.}
P2P and Per-Pcs are included in the full-history main comparison but not in
this controlled cold-start study, because our reproduced pipelines use
user representations precomputed from complete histories. We restrict this
study to baselines whose history budgets can be varied within the same
reproduced protocol. MTA is omitted on LaMP-2N because no complete released
configuration is available.

\subsection{Additional Efficiency and Scalability Results}
\label{app:efficiency}

\begin{table}[t]
\setlength{\tabcolsep}{4.2pt}
\renewcommand{\arraystretch}{1.12}

\caption{
Shared and per-user adaptation parameters and bf16 storage on LaMP-2N and LaMP-3. User columns report persistent state per user.
}
\label{tab:efficiency-accounting}
\begin{center}
\begin{tabular}{@{}lrrrr@{}}
\toprule
\textbf{Method}
& \shortstack{\textbf{Shared}\\\textbf{params}}
& \shortstack{\textbf{User}\\\textbf{params}}
& \shortstack{\textbf{Shared}\\\textbf{storage}}
& \shortstack{\textbf{User}\\\textbf{storage}} \\
\midrule

\multicolumn{5}{@{}l}{\textit{(a) LaMP-2N}} \\

OPPU
& 6.29\,M
& 4.19\,M
& 12.58\,MB
& 8.39\,MB \\

PerFit
& 1.18\,M
& 0.79\,M
& 2.36\,MB
& 1.57\,MB \\

P2P
& 307.66\,M
& 0
& 615.33\,MB
& 0 \\

\rowcolor{cyan!10}
\textbf{LINEUP}
& 75.61\,M
& 8
& 151.23\,MB
& 16\,B \\

\midrule

\multicolumn{5}{@{}l}{\textit{(b) LaMP-3}} \\

Per-Pcs
& 454.03\,M
& 8.39\,M
& 908.07\,MB
& 16.78\,MB \\

MTA
& 999.42\,M
& 9.99\,M
& 1{,}998.85\,MB
& 19.99\,MB \\

\rowcolor{cyan!10}
\textbf{LINEUP}
& 75.58\,M
& 8
& 151.16\,MB
& 16\,B \\

\bottomrule
\end{tabular}
\end{center}
\end{table}

\paragraph{Storage accounting.}
For the evaluated configurations, we analytically compute storage curves
from shared adaptation parameters $P_{\mathrm{shared}}$ and persistent
parameters per user $P_{\mathrm{user}}$ as
$S(N)=2(P_{\mathrm{shared}}+NP_{\mathrm{user}})$ bytes for $N$ users.
Storage is normalized to bf16, with 1\,MB $=10^6$ bytes and
1\,GB $=10^9$ bytes. Here $N$ is a deployment-population variable,
independent of the evaluated target-user cohort; shared parameters
and per-user state sizes remain fixed.
The per-user term counts parameter state retained for reuse, whether
fitted or assembled. We exclude the common frozen Llama-2-7B backbone
and auxiliary pretrained encoders. LINEUP's shared count includes its
task LoRA and fixed feature projections. P2P's hypernetwork is counted as shared state, while its query-generated
LoRAs are transient and excluded from the per-user term.
Per-Pcs retains an assembled personalized LoRA, and MTA retains
a trained stacked LoRA; reconstructable merged copies are not counted
separately. Table~\ref{tab:efficiency-accounting} summarizes the parameter
and storage counts.

\paragraph{Comparisons.}
The displayed curves begin at 100 users, reflecting the approximately
100 target users in most tasks under our evaluation protocol.
For LaMP-2N, LINEUP becomes more storage-efficient than OPPU at
17 users and PerFit at 95 users, while remaining below P2P's
615.33\,MB shared state over the displayed range.
At LaMP-2N's actual evaluation size of 49 target users, PerFit therefore
has lower total adaptation storage than LINEUP.
We additionally compare Per-Pcs and MTA using their complete LaMP-3
configurations (Figure~\ref{fig:coldstart_efficiency_supp}(c)).
At 1,024 users, the projected totals are approximately 151.18\,MB
for LINEUP, 18.09\,GB for Per-Pcs, and 22.47\,GB for MTA.
Thus, directly sharing personalization capacity avoids the substantial
per-user storage that remains when reusable adapter components are
assembled into large individual modules.

\subsection{Ablation Protocol}
\label{app:ablation_protocol}

We evaluate all variants on LaMP-2M and LaMP-5 over three seeds
and report mean $\pm$ standard deviation.
The two tasks cover different routing regimes:
LaMP-2M uses $M_c=M_a=4$, whereas LaMP-5 uses $M_c=16$ and $M_a=4$.
Query calibration therefore reweights recalled factors on LaMP-2M
and additionally selects the active subset on LaMP-5.

All variants inherit the corresponding task configuration, optimization
settings, data splits, evaluator, and training schedule without re-tuning.
The training horizon and reference checkpoint are 5,000 and 4,096 episodes
for LaMP-2M, and 500 and 400 for LaMP-5.
Full LINEUP and S1 only reuse the corresponding main-experiment runs;
the two routing ablations are trained under the same schedule as S1-only
and evaluated at the corresponding reference checkpoint.

\paragraph{Random candidates.}
We replace the history-guided candidate set with a uniformly sampled
subset of the same size:
\begin{equation}
\mathcal C_u^{\mathrm{rand}}
\sim
\operatorname{UniformSubset}\bigl(\{1,\ldots,K\},M_c\bigr).
\end{equation}
The subset is generated deterministically from the experimental seed
and user identifier, and remains fixed across that user's queries
and batches. The original score $z_{u,x}=s_u+\lambda W_Q\tilde e_x$
still determines active-factor selection and mixture weights within
this set. Thus, this control removes history-guided retrieval rather
than all history information.

\paragraph{No query calibration.}
We set $\lambda=0$, so $z_{u,x}=s_u$, throughout training and inference.
History-conditioned candidate recall, optimization settings, and all
remaining model components are unchanged.

\paragraph{Removing the user-specific correction.}
S1 only retains shared-factor composition and omits the entire S2
correction mechanism. Full LINEUP adds history-initialized, support-fitted
user-code correction to the same S1 model, so this comparison evaluates
the contribution of S2 as a whole.

\paragraph{Initialized code without support fitting.}
\emph{S1 + initialized code} uses the same shared checkpoints and
available user history as Full LINEUP for each seed. It retains S1 and
the shared correction maps, applying the initialized code
$c_u^{(0)}=G\tilde e_u$ across the user's queries without target-user
optimization. Comparison with S1-only evaluates the shared maps and
history-based initialization together; comparison with Full LINEUP
isolates the additional contribution of support fitting.

\begin{table}[t]
\setlength{\tabcolsep}{4.2pt}
\renewcommand{\arraystretch}{1.12}

\caption{
Sensitivity of LINEUP to user-code dimension $r_c$ on LaMP-2M and LaMP-5. \underline{Underline} and \textbf{bold} denote the second-best and best means, including ties; $\dagger$ marks the default configuration. Performance improves from $r_c=2$ to $8$, while larger codes yield no consistent gains across metrics.
}
\label{tab:code_dimension_sensitivity}
\begin{center}
\begin{tabular}{@{}ccccc@{}}
\toprule
& \multicolumn{2}{c}{\textbf{LaMP-2M}}
& \multicolumn{2}{c}{\textbf{LaMP-5}} \\
\cmidrule(lr){2-3}
\cmidrule(lr){4-5}

$\boldsymbol{r_c}$
& \textbf{Acc} $\uparrow$
& \textbf{F1} $\uparrow$
& \textbf{R-1} $\uparrow$
& \textbf{R-L} $\uparrow$ \\
\midrule

$2$
& $.633 \pm .002$
& $.530 \pm .002$
& $.551 \pm .003$
& $.486 \pm .001$ \\

$4$
& $.641 \pm .002$
& $.535 \pm .002$
& $\underline{.563} \pm .002$
& $\underline{.491} \pm .001$ \\

\rowcolor{cyan!10}
$8^{\dagger}$
& $.645 \pm .003$
& $\underline{.539} \pm .002$
& $\mathbf{.564} \pm .003$
& $\mathbf{.494} \pm .004$ \\

$16$
& $\underline{.647} \pm .004$
& $\mathbf{.541} \pm .001$
& $\mathbf{.564} \pm .002$
& $\mathbf{.494} \pm .002$ \\

$32$
& $\mathbf{.653} \pm .006$
& $.538 \pm .003$
& $.561 \pm .002$
& $.490 \pm .002$ \\

\bottomrule
\end{tabular}
\end{center}
\end{table}

\subsection{User-Code Dimension Sensitivity}
\label{app:code_dimension_sensitivity}

\paragraph{Protocol.}
We sweep $r_c\in\{2,4,8,16,32\}$ on LaMP-2M and LaMP-5, following
the data splits and evaluation protocol in
Appendix~\ref{app:ablation_protocol}.
For each task and seed, the task-adapted backbone and S1 checkpoint
are fixed across dimensions. For each new dimension, we retrain
the shared correction maps $\{V_\ell\}$ and history initializer $G$;
$r_c=8$ reuses the existing Full LINEUP runs.
All remaining settings are unchanged, including S2 training budgets,
checkpoint selection, support construction, learning rates, and
regularization. Each target user initializes one code from available
history and fits only this code for ten Adam steps, keeping all shared
parameters fixed.

\paragraph{Results.}
Table~\ref{tab:code_dimension_sensitivity} shows that increasing $r_c$
from 2 to 8 improves mean performance on all four metrics.
Moving from 8 to 16 raises both LaMP-2M means by $0.002$, while
LaMP-5 scores remain unchanged at the reported precision.
At $r_c=32$, LaMP-2M Accuracy reaches its highest mean, but Macro-F1
and both ROUGE means decrease relative to $r_c=16$.
Larger codes therefore do not yield consistent improvements across
metrics. The default $r_c=8$ provides an effective compact configuration,
matching the best reported generation means and remaining close to
$r_c=16$ on classification.

\clearpage

\begin{table}[t]
\small
\setlength{\tabcolsep}{2.5pt}
\renewcommand{\arraystretch}{1.12}

\caption{
Cold-start performance on LaMP-2N across history budgets.
\textbf{Bold} denotes the best reported mean for each metric and budget.
}
\label{tab:coldstart_lamp2n}
\begin{center}
\begin{tabular}{@{}lccccccc@{}}
\toprule
\textbf{Method}
& $\mathbf{k=1}$
& $\mathbf{2}$
& $\mathbf{4}$
& $\mathbf{8}$
& $\mathbf{16}$
& $\mathbf{32}$
& \textbf{Full} \\
\midrule

\multicolumn{8}{@{}l}{\textbf{Acc} $\uparrow$} \\

Basic
& $.783 \pm .002$
& $.783 \pm .002$
& $.783 \pm .002$
& $.783 \pm .002$
& $.783 \pm .002$
& $.783 \pm .002$
& $.783 \pm .002$ \\

OPPU
& $.793 \pm .002$
& $.793 \pm .002$
& $.793 \pm .002$
& $.793 \pm .002$
& $.793 \pm .002$
& $.793 \pm .002$
& $.801 \pm .001$ \\

PerFit
& $.679 \pm .000$
& $.711 \pm .001$
& $.710 \pm .001$
& $.713 \pm .001$
& $.722 \pm .001$
& $.724 \pm .001$
& $.818 \pm .000$ \\

RAG
& $.750 \pm .001$
& $.756 \pm .001$
& $.768 \pm .001$
& $.770 \pm .001$
& $.778 \pm .001$
& $.788 \pm .001$
& $.817 \pm .001$ \\

\rowcolor{cyan!10}
\textbf{LINEUP}
& $\mathbf{.806} \pm .006$
& $\mathbf{.807} \pm .007$
& $\mathbf{.807} \pm .005$
& $\mathbf{.807} \pm .005$
& $\mathbf{.806} \pm .005$
& $\mathbf{.813} \pm .005$
& $\mathbf{.822} \pm .005$ \\

\midrule

\multicolumn{8}{@{}l}{\textbf{F1} $\uparrow$} \\

Basic
& $.536 \pm .002$
& $.536 \pm .002$
& $.536 \pm .002$
& $.536 \pm .002$
& $.536 \pm .002$
& $.536 \pm .002$
& $.536 \pm .002$ \\

OPPU
& $.536 \pm .002$
& $.536 \pm .002$
& $.536 \pm .002$
& $.536 \pm .002$
& $.537 \pm .002$
& $.537 \pm .003$
& $.556 \pm .002$ \\

PerFit
& $.457 \pm .002$
& $.484 \pm .001$
& $.489 \pm .002$
& $.510 \pm .001$
& $.514 \pm .001$
& $.518 \pm .002$
& $.589 \pm .002$ \\

RAG
& $.507 \pm .002$
& $.514 \pm .003$
& $.537 \pm .004$
& $.554 \pm .004$
& $.571 \pm .003$
& $\mathbf{.588} \pm .001$
& $.594 \pm .001$ \\

\rowcolor{cyan!10}
\textbf{LINEUP}
& $\mathbf{.560} \pm .005$
& $\mathbf{.564} \pm .005$
& $\mathbf{.575} \pm .004$
& $\mathbf{.585} \pm .004$
& $\mathbf{.584} \pm .003$
& $\mathbf{.588} \pm .003$
& $\mathbf{.621} \pm .003$ \\

\bottomrule
\end{tabular}
\end{center}
\end{table}

\begin{table}[t]
\small
\setlength{\tabcolsep}{2.5pt}
\renewcommand{\arraystretch}{1.12}

\caption{
Cold-start performance on LaMP-4 across history budgets.
\textbf{Bold} denotes the best reported mean for each metric and budget.}
\label{tab:coldstart_lamp4}
\begin{center}
\begin{tabular}{@{}lccccccc@{}}
\toprule
\textbf{Method}
& $\mathbf{k=1}$
& $\mathbf{2}$
& $\mathbf{4}$
& $\mathbf{8}$
& $\mathbf{16}$
& $\mathbf{32}$
& \textbf{Full} \\
\midrule

\multicolumn{8}{@{}l}{\textbf{R-1} $\uparrow$} \\

Basic
& $.186 \pm .000$
& $.186 \pm .000$
& $.186 \pm .000$
& $.186 \pm .000$
& $.186 \pm .000$
& $.186 \pm .000$
& $.186 \pm .000$ \\

OPPU
& $.215 \pm .001$
& $.215 \pm .000$
& $.215 \pm .000$
& $.215 \pm .000$
& $.215 \pm .000$
& $.215 \pm .001$
& $.216 \pm .000$ \\

PerFit
& $.202 \pm .001$
& $.202 \pm .001$
& $.202 \pm .000$
& $.203 \pm .001$
& $.203 \pm .001$
& $.203 \pm .001$
& $.207 \pm .001$ \\

RAG
& $.192 \pm .001$
& $.192 \pm .002$
& $.192 \pm .001$
& $.192 \pm .000$
& $.192 \pm .001$
& $.193 \pm .001$
& $.198 \pm .001$ \\

MTA
& $.197 \pm .000$
& $.198 \pm .000$
& $.200 \pm .000$
& $.200 \pm .000$
& $.202 \pm .000$
& $.202 \pm .001$
& $.218 \pm .001$ \\

\rowcolor{cyan!10}
\textbf{LINEUP}
& $\mathbf{.219} \pm .001$
& $\mathbf{.219} \pm .000$
& $\mathbf{.221} \pm .000$
& $\mathbf{.222} \pm .001$
& $\mathbf{.222} \pm .001$
& $\mathbf{.222} \pm .000$
& $\mathbf{.222} \pm .001$ \\

\midrule

\multicolumn{8}{@{}l}{\textbf{R-L} $\uparrow$} \\

Basic
& $.161 \pm .001$
& $.161 \pm .001$
& $.161 \pm .001$
& $.161 \pm .001$
& $.161 \pm .001$
& $.161 \pm .001$
& $.161 \pm .001$ \\

OPPU
& $.191 \pm .001$
& $.191 \pm .001$
& $.191 \pm .001$
& $.191 \pm .001$
& $.191 \pm .001$
& $.192 \pm .001$
& $.192 \pm .001$ \\

PerFit
& $.179 \pm .000$
& $.179 \pm .001$
& $.179 \pm .000$
& $.180 \pm .000$
& $.180 \pm .001$
& $.181 \pm .001$
& $.183 \pm .001$ \\

RAG
& $.170 \pm .001$
& $.170 \pm .001$
& $.170 \pm .000$
& $.170 \pm .001$
& $.170 \pm .001$
& $.170 \pm .001$
& $.175 \pm .000$ \\

MTA
& $.176 \pm .000$
& $.177 \pm .000$
& $.179 \pm .000$
& $.180 \pm .000$
& $.181 \pm .000$
& $.181 \pm .000$
& $.195 \pm .000$ \\

\rowcolor{cyan!10}
\textbf{LINEUP}
& $\mathbf{.197} \pm .000$
& $\mathbf{.197} \pm .000$
& $\mathbf{.197} \pm .001$
& $\mathbf{.197} \pm .001$
& $\mathbf{.197} \pm .001$
& $\mathbf{.197} \pm .001$
& $\mathbf{.197} \pm .001$ \\

\bottomrule
\end{tabular}
\end{center}
\end{table}